\documentclass{article}

\usepackage{microtype}
\usepackage{graphicx}
\usepackage{booktabs} %

\usepackage{hyperref,flushend}

\usepackage[accepted]{icml2024}

\usepackage{amsmath}
\usepackage{amssymb}
\usepackage{mathtools}
\usepackage{amsthm}

\usepackage[capitalize,noabbrev]{cleveref}

\theoremstyle{plain}
\newtheorem{theorem}{Theorem}[section]
\newtheorem{proposition}[theorem]{Proposition}
\newtheorem{lemma}[theorem]{Lemma}
\newtheorem{corollary}[theorem]{Corollary}
\theoremstyle{definition}

\theoremstyle{remark}

\usepackage[textsize=tiny]{todonotes}

\usepackage{subcaption}
\usepackage{wrapfig}

\usepackage{amsmath,amsfonts,bm}

\def\eqref#1{equation~\ref{#1}}

\def\floor#1{\lfloor #1 \rfloor}
\def\1{\bm{1}}

\def\mI{{\bm{I}}}

\DeclareMathAlphabet{\mathsfit}{\encodingdefault}{\sfdefault}{m}{sl}
\SetMathAlphabet{\mathsfit}{bold}{\encodingdefault}{\sfdefault}{bx}{n}

\newcommand{\E}{\mathbb{E}}

\newcommand{\R}{\mathbb{R}}

\DeclareMathOperator*{\argmax}{arg\,max}
\DeclareMathOperator*{\argmin}{arg\,min}

\DeclareMathOperator{\Tr}{Tr}

\newcommand{\ie}{i.e.}
\newcommand{\generror}[2]{\mathrm{gen}(#1, #2)}

\newcommand{\MI}[2]{\mathsf{I}(#1 ; #2)}
\newcommand{\CMI}[3]{\mathsf{I}(#1 ; #2 | #3)}

\newcommand{\kSMI}[3]{\mathsf{SI}_{#1}^{(1)}(#2 ; #3)}
\newcommand{\DMI}[3]{\mathsf{I}^{#3}(#1 ; #2)}

\newcommand{\KL}[2]{\mathbf{KL}(#1 \| #2)}
\newcommand{\risk}[1]{\mathcal{R}(#1)}
\newcommand{\riskn}[1]{\widehat{\mathcal{R}}_n(#1)}
\newcommand{\drisk}[2]{\mathcal{R}^{#2}(#1)}
\newcommand{\driskn}[2]{\widehat{\mathcal{R}}^{#2}_n(#1)}

\newcommand{\setW}{\mathrm{W}}
\newcommand{\setZ}{\mathrm{Z}}

\newcommand{\sti}[0]{\mathrm{St}}

\renewcommand*{\eqref}[1]{(\ref{#1})}

\icmltitlerunning{Slicing Mutual Information Generalization Bounds for Neural Networks}

\begin{document}

\twocolumn[
\icmltitle{Slicing Mutual Information Generalization Bounds for Neural Networks}

\icmlsetsymbol{equal}{*}

\begin{icmlauthorlist}
\icmlauthor{Kimia Nadjahi}{mit}
\icmlauthor{Kristjan Greenewald}{ibm}
\icmlauthor{Rickard Br\"uel Gabrielsson}{mit}
\icmlauthor{Justin Solomon}{mit}
\end{icmlauthorlist}

\icmlaffiliation{mit}{MIT}
\icmlaffiliation{ibm}{MIT-IBM Watson AI Lab; IBM Research}

\icmlcorrespondingauthor{Kimia Nadjahi}{kimia.nadjahi@ens.fr}

\icmlkeywords{Machine Learning, ICML}

\vskip 0.3in
]

\printAffiliationsAndNotice{}  %

\begin{abstract}
    The ability of machine learning (ML) algorithms to generalize well to unseen data has been studied through the lens of information theory, by bounding the generalization error with the input-output mutual information (MI), \ie, the MI between the training data and the learned hypothesis. Yet, these bounds have limited practicality for modern ML applications (e.g., deep learning), due to the difficulty of evaluating MI in high dimensions. Motivated by recent findings on the compressibility of neural networks, we consider algorithms that operate by \emph{slicing} the parameter space, \ie, trained on random lower-dimensional subspaces. We introduce new, tighter information-theoretic generalization bounds tailored for such algorithms, demonstrating that slicing improves generalization. Our bounds offer significant computational and statistical advantages over standard MI bounds, as they rely on scalable alternative measures of dependence, \ie, disintegrated mutual information and $k$-sliced mutual information. Then, we extend our analysis to algorithms whose parameters do not need to exactly lie on random subspaces, by leveraging rate-distortion theory. This strategy yields generalization bounds that incorporate a distortion term measuring model compressibility under slicing, thereby tightening existing bounds without compromising performance or requiring model compression. Building on this, we propose a regularization scheme enabling practitioners to control generalization through compressibility. Finally, we empirically validate our results and achieve the computation of non-vacuous information-theoretic generalization bounds for neural networks, a task that was previously out of reach. 
\end{abstract}

\section{Introduction}

Generalization is a fundamental aspect of machine learning, where models optimized on training data are expected to perform similarly on test data. %
Neural networks (NNs), in particular, are able to both perform and generalize well, allowing them to achieve excellent test performance on complex tasks. Despite this empirical success, the architectural factors influencing how well NNs generalize are not fully understood theoretically. This has motivated a substantial body of work using a %
variety of tools to bound their \emph{generalization error} \citep{jiangfantastic}, \ie., the gap between the average loss on training data (\emph{empirical risk}) and its expected loss on a new data \emph{from the same distribution} (\emph{population risk}). The goal is to identify when and why a given model yields a low generalization error, and ultimately, design architectures or training algorithms that guarantee good generalization. Common approaches include PAC-Bayes analysis \citep{dziugaite2017computing} and information theory \citep{xu2017}.

Compression is another topic which has provided a fertile ground for machine learning research. As model architectures have become more and more complex, their evaluation, training and fine-tuning become even more challenging. For instance, large language models are parameterized with billions of parameters. Compressed models, which reduce the number of trainable parameters without significantly deteriorating the performance, have been growing in practical relevance, for example LoRA finetuning of large language models \cite{hu2021lora}. One compression scheme which has found success consists in training NNs on random, lower-dimensional subspaces. NNs compressed that way have been shown to yield satisfying test performances in various tasks while being faster to train \cite{li2018measuring}. This framework has recently been applied by \citet{lotfi2022pac} to significantly improve PAC-Bayes generalization bounds, to the point where they closely match empirically observed generalization error.

A recent line of work argues that there is actually an interplay between \emph{compressible} models and their ability to generalize well. The main conclusion is that one can construct tighter generalization bounds by leveraging compression schemes \cite{arora2018stronger,hsu2021generalization,kuhn2021robustness,sefidgaran2022rate}.

In this paper, we seek further understanding on the generalization ability of learning algorithms trained on random subspaces. %
We introduce new information-theoretic generalization bounds tailored for such algorithms, which are tighter than existing ones. Our bounds demonstrate that algorithms that are ``compressible'' via random slicing have significantly better information-theoretic generalization guarantees. We also find an intriguing connection to Sliced Mutual Information \cite{goldfeld2021sliced,goldfeld2022}, which we explore in learning problems where the information-theoretic generalization bounds are analytically computable. We then leverage the computational and statistical benefits of our sliced approach to empirically compute nonvacuous information-theoretic generalization bounds for various neural networks. 

We further increase the practicality of our approach by using the \emph{rate-distortion}-based framework introduced by \citet{sefidgaran2022rate} to extend our bounds to the setting where the weight vector $W$ only approximately lies on random subspace. This extension applies when the loss is Lipschitz w.r.t.\ the weights, which we can promote using techniques by \citet{bethune2023}. As \citet{sefidgaran2022rate} did for quantization, this allows us to apply generalization bounds based on projection and quantization to networks whose weights are unrestricted. We tighten the bound by using regularization in training to encourage the weights to be close to the random subspace.

\section{Related Work} \label{sec:relatedwork}

\textbf{Compression of neural networks.} Our work focuses on random projection and quantization \cite{hubara2016binarized} as tools for compressing neural networks. Many other compression approaches exist \cite{cheng2017survey}, e.g., pruning \cite{dong2017more,MLSYS2020_d2ddea18}, low-rank compression \cite{wen2017coordinating}, and optimizing architectures via neural architecture search and meta-learning \cite{pham2018efficient,cai2019once,finn2017model}. Further exploring alternative compression approaches from an information-theoretic generalization bound perspective is an interesting avenue for future work.

\textbf{Compressibility and generalization.} A body of work has emerged leveraging various notions of compressibility to adequately explain why neural networks can generalize \cite{arora2018stronger,Suzuki2020Compression,simsekli2020hausdorff,bu2020,hsu2021generalization,kuhn2021robustness,sefidgaran2022rate}. In particular, \citet{bu2020} and \citet{sefidgaran2022rate} connected compressibility and generalization using the rate-distortion theory, which inspired our analysis. \citet{sefidgaran2022rate} derived a set of theoretical generalization bounds, but their applicability to compressible neural networks is unclear. \citet{bu2020} established a generalization bound for a learning model whose weights $W$ are optimized and then compressed into $\hat{W}$. They consider compression as a post-processing technique, while we propose to take compressibility into account during training. Furthermore, \citet{bu2020} applied the rate-distortion theory for a slightly different purpose than ours: to compare the population risk of the compressed model with that of the original model. In contrast, we use it to bound the generalization error of the original model and show that if it is almost compressible on a random subspace, one can obtain significantly tighter bounds than existing information-theoretic ones. 

\textbf{Conditional MI generalization bounds.} Following %
\cite{xu2017} and \cite{bu2019}, which treat the training data as random, \cite{steinke2020reasoning} instead obtain a bound where the dataset is fixed (\ie~\emph{conditioned} on a dataset). This framework assumes that two independent datasets are available, and random Bernoulli indicator variables create a random training set by randomly selecting which of the two datasets to use for the $i$th training point. This approach has the advantage of creating a generalization bound involving the mutual information between the learned weights and a set of \emph{discrete} random variables, in which case the mutual information is always finite. Connections to other generalization bound strategies and to data privacy are established by \cite{steinke2020reasoning}. Followup works tightened these bounds by considering the conditional mutual information between the indicator variables and either the \emph{predictions} \citep{harutyunyan2021information,haghifam2022} or \emph{loss} \citep{wang2023tighter} of the learned model rather than the weights. {A practical limitation of this general approach is that it requires a second dataset (or \emph{supersample}) to compute the conditional mutual information, whereas this extra data could be used to get a better estimate of the test error (hence, the generalization error) directly. Additionally, some of these bounds depend on a mutual information term between low-dimensional variables (e.g., functional CMI-based bounds \citep{harutyunyan2021information}), which can be evaluated efficiently but does not inform practitioners for selecting model architectures.} Exploring slicing for the conditional MI framework is beyond the scope of our paper and is a promising direction for future work.

\textbf{Other generalization bounds for neural networks.} Beyond the information-theoretic frameworks above, many methods bound the generalization of neural networks. Classic approaches in learning theory bound generalization error with complexity of the hypothesis class \citep{bartlett2002rademacher,vapnik2015uniform}, but these fail to explain the generalization ability of highly flexible deep neural network models \citep{zhang2016understanding}. More successful approaches include the PAC-Bayes framework (including \citeauthor{lotfi2022pac}, whose use of slicing inspired our work), margin-based approaches \citep{koltchinskii2002empirical, kuznetsov2015rademacher,chuang2021measuring}, flatness of the loss curve \cite{petzka2021relative}, and even empirically-trained prediction not based on theoretical guarantees \citep{jiang2020neurips,lassance2020ranking, natekar2020representation, schiff2021gi}. Each approach has its own benefits and drawbacks; for instance, many of the tightest predictions are highly data-driven and as a result may provide limited insight into the underlying sources of generalization and how to design networks to promote it. %

\textbf{Our work.} Our approach dramatically improves the tightness of \emph{input-output information-theoretic generalization bounds} for neural networks, which up to this point have not seen practical use. That said, our bounds (unsurprisingly) are still looser than generalization bounds available through some other frameworks, particularly those employing additional data (e.g., data-driven PAC-Bayes priors \citep{lotfi2022pac} or the super-sample of conditional MI bounds \citep{wang2023tighter}) or involving some kind of trained or ad hoc prediction function. Regardless, continuing to improve information-theoretic bounds is a fruitful endeavor that improves our understanding of the connection between machine learning and information theory, and gives insights that can drive algorithmic and architectural innovation.

\section{Preliminaries}
\label{subsec:notations}

Let $\setZ$ be the input data space, $\setW \subseteq \R^D$ the hypothesis space, and $\ell : \setW \times \setZ \to \R_+$ a loss function. For instance, in supervised learning, $\setZ = \{(x, y) \in \mathrm{X} \times \{-1, 1\}\}$ is the set of feature-label pairs, $w \in \setW$ is the parameter vector of a classifier $f_w : \R^D \to \{-1, 1\}$ (e.g., the weights of a neural network), and $\ell(w, (x,y)) = \1_{y \neq f_w(x)}$ is the error made by predicting $y$ as $f_w(x)$. 

Consider a training dataset $S_n \triangleq (Z_1, \dots, Z_n) \in \setZ^n$ consisting of $n$ i.i.d. samples from $\mu$. For any $w \in \setW$, let $\risk{w} \triangleq \E_{Z \sim \mu}[\ell(w, Z)]$ denote the \emph{population risk}, and $\riskn{w} \triangleq \frac1n \sum_{i=1}^n \ell(w, z_i)$ the \emph{empirical risk}. Training a machine learning algorithm aims at minimizing the population risk, \ie, solving $\min_{w \in \setW} \risk{w}$. However, computing $\risk{w}$ is difficult in practice, since the data distribution $\mu$ is generally unknown: one would only observe a finite number of samples from $\mu$. Therefore, a common workaround is \emph{empirical risk minimization}, \ie, $\min_{w \in \setW} \riskn{w}$. A learning algorithm can then be described as the mapping $\mathcal{A} : \setZ^n \to \setW$, where $\mathcal{A}(S_n)$ is a hypothesis learned from $S_n$. We assume that $\mathcal{A}$ is \emph{randomized}: its output $W \triangleq \mathcal{A}(S_n)$ is a random variable distributed from $P_{W|S_n}$. 

The \emph{generalization error} of $\mathcal{A}$ is defined as %
$\generror{\mu}{\mathcal{A}} \triangleq \E_{P_{W|S_n} \otimes \mu^{\otimes n}}[\risk{W} - \riskn{W}]$, where the expectation $\E$ is taken with respect to (w.r.t.) the joint distribution of $(W, S_n)$. %
The higher $\generror{\mu}{\mathcal{A}}$, the more $\mathcal{A}$ overfits when trained on $S_n \sim \mu^{\otimes n}$.

\subsection{Information-theoretic generalization bounds}

{In recent years, there has been a flurry of interest in using theoretical approaches to bound $\generror{\mu}{\mathcal{A}}$ using \emph{mutual information} (MI). }
The most common information-theoretic bound on generalization error was introduced by \citet{xu2017} and depends on the mutual information between the training data $S_n$ and the hypothesis $W$ learned from $S_n$. We recall the formal statement below.%
\begin{theorem}[\citealp{xu2017}] \label{thm:xu}
Assume that $\ell(w, Z)$ is $\sigma$-sub-Gaussian\footnote{A random variable $X$ is $\sigma$-sub-Gaussian ($\sigma > 0$) under $\mu$ if for $t \in \R$, $\E_{\mu}[e^{t(X - \E_{\mu}[X])}] \leq e^{\sigma^2 t^2 / 2}$.} under $Z \sim \mu$ for all $w \in \mathrm{W}$. Then, $|\generror{\mu}{\mathcal{A}}| \leq \sqrt{2\sigma^2~\MI{W}{S_n}/n}$, where $\MI{W}{S_n}$ is the mutual information between $W = \mathcal{A}(S_n)$ and $S_n$.
\end{theorem}
The class of $\sigma$-sub-Gaussian losses includes Gaussian-distributed losses $\ell(w, Z) \sim \mathcal{N}(0, \sigma^2)$ and bounded losses satisfying $0\leq \ell(w, Z) \leq 2\sigma$.
For example, the 0-1 correct classification loss satisfies this with $\sigma = 0.5$.
Subsequently, \citet{bu2019} used the averaging structure of the empirical loss to derive a bound that depends on $\MI{W}{Z_i}$. Evaluating MI on each \emph{individual} data point $Z_i$ rather than the entire training dataset $S_n$ has been shown to produce tighter bounds than \citet{xu2017} \citep[{\S}IV]{bu2019}. 
\begin{theorem}[\citealp{bu2019}]
\label{thm:bu}
Assume that $\ell(\tilde{W}, \tilde{Z})$ is $\sigma$-sub-Gaussian under $(\tilde{W}, \tilde{Z}) \sim P_W \otimes \mu$. Then,
$|\generror{\mu}{\mathcal{A}}| \leq (1/n) \sum_{i=1}^n \sqrt{2\sigma^2~\MI{W}{Z_i}}$, where $\MI{W}{Z_i}$ is the mutual information between $W = \mathcal{A}(S_n)$ and $Z_i$.
\end{theorem}

These and other information-theoretic bounds, however, suffer from the fact that the dimension of $W$ can be large when using modern ML models, e.g.\ NNs. Indeed, the sample complexity of MI estimation scales poorly with dimension \citep{paninski2003estimation}. Collecting more samples of $(W, Z_i)$ can be expensive, especially with NNs, as one realization of $W \sim P_{W|S_n}$ requires one complete training run. Moreover, \citet{mcallester20a} recently proved that estimating MI from finite data have important statistical limitations when the underlying MI is large, e.g., hundreds of bits.

\subsection{Random subspace training and sliced mutual information}

{While modern neural networks use large numbers of parameters, common architectures can be highly compressible by \emph{random slicing}: \citet{li2018measuring} found that restricting $W \in \R^D$ during training to lie in a $d$-dimensional subspace spanned by a random matrix (with $d \ll D$) not only provides computational advantages, but does not meaningfully damage the performance of the neural network, for appropriate choice of $d$ (often two orders of magnitude smaller than $D$). They interpreted this fact as indicating \emph{compressibility} of the neural network architecture up to some \emph{intrinsic dimension} $d$, below which performance degrades.}%

Denote by $\sti(d,D) = \{ \Theta \in \R^{D \times d}\,:\,\Theta^\top \Theta = \mI_d \}$ the Stiefel manifold, equipped with the uniform distribution $P_\Theta$. We consider a learning algorithm $\mathcal{A}^{(d)}$ whose hypothesis space is restricted to $\setW_{\Theta,d} \triangleq \{ w \in \R^D\,:\,\exists w' \in \R^d\; \mathrm{s.t.}\;w = \Theta w' \}$, where $\Theta \sim P_\Theta$. Note that $\Theta$ is \emph{not} trained: it is randomly generated from $P_\Theta$ at the beginning of training, and frozen during training. In other words, $\mathcal{A}^{(d)}$ trains the parameters on a random $d$-dimensional subspace of $\R^D$ characterized by $\Theta$. %

{\textbf{Sliced mutual information.} %
The random $d$-dimensional weight subspace is closely related to \emph{slicing}, which projects a high dimensional quantity to a random lower dimensional subspace.
Intriguingly, a recent line of work has considered slicing mutual information, yielding significant sample complexity and computational advantages in high-dimensional regimes. \citet{goldfeld2021sliced} and \citet{goldfeld2022} slice the arguments of MI via random $k$-dimensional projections, thus defining the \emph{$k$-Sliced Mutual Information} (SMI). $\mathsf{SI}_k$ has been shown to retain many important properties of MI \citep{goldfeld2022}, and more importantly, the statistical convergence rate for estimating $\mathsf{SI}_k(X;Y)$ depends on $k$ but not the ambient dimensions $d_x, d_y$. This provides significant advantages over MI, whose computation generally requires an exponential number of samples in $\max(d_x, d_y)$ \citep{paninski2003estimation}. Similar convergence rates can be achieved while slicing in only one dimension. %
Recently, \citet{wongso2023using} empirically connected generalization to SMI %
between the true class labels $Y$ and the hidden representations $T$ of NNs.
}

\section{Information-Theoretic Generalization Bounds for Compressed Models} \label{sec:bounds}

Motivated by the advantageous properties of sliced mutual information and the practical success of training neural networks in random subspaces, we seek an input-output information-theoretic generalization bound for learning algorithms trained on the random subspace $\mathrm{W}_{\Theta, d}$. We will see that improved tightness and statistical properties of such bounds allow these to scale to larger models than possible for the traditional bounds in \Cref{thm:xu,thm:bu}, without significantly damaging the test-time performance of the resulting learned models. To this end, we will derive in this section new information-theoretic bounds on the generalization error for these random subspace algorithms. Using the terminology in \Cref{subsec:notations}, the generalization error of $\mathcal{A}^{(d)}$ is 
\begin{equation}
    \generror{\mu}{\mathcal{A}^{(d)}} = \E_{P_{W' | \Theta, S_n} \otimes P_\Theta \otimes \mu^{\otimes n}}[\risk{\Theta W'} - \riskn{\Theta W'}] \,.
\end{equation}
Note that the expectation is taken w.r.t. to $P_\Theta$, so the error does not depend on $\Theta$. %

A natural strategy to bound the generalization error of $\mathcal{A}^{(d)}$ is by applying \citet{xu2017}: if $\ell(w, Z)$ is $\sigma$-sub-Gaussian under $Z \sim \mu$ for all $(\Theta, w) \in \sti(d, D) \times \mathrm{W}_{\Theta, d}$, then by \Cref{thm:xu}, 
\begin{equation} \label{eq:application_xu}
    | \generror{\mu}{\mathcal{A}^{(d)}} | \leq \sqrt{\frac{2 \sigma^2}{n} \MI{\Theta W'}{S_n}} \,,    
\end{equation} 
where $\Theta W' = \mathcal{A}^{(d)}(S_n)$, $\Theta \sim P_\Theta$. However, this bound does not clearly explain the impact of the intrinsic dimension $d$ on generalization. In addition, the MI term $\MI{\Theta W'}{S_n}$ is hard to estimate in modern machine learning applications since $\Theta W'$ is high-dimensional.

We derive new upper-bounds on $\generror{\mu}{\mathcal{A}^{(d)}}$ to mitigate these issues. Our strategy consists in applying the \emph{disintegration technique} on the hypothesis space $\setW_{\Theta, d}$. %
In our setting, disintegration boils down to deriving a bound for a fixed $\Theta$, then taking the expectation over $P_\Theta$. This yields information-theoretic bounds which are tighter than existing ones and rely on mild assumptions. Moreover, our bounds exhibit an explicit dependence on $d$, which helps capture the impact of compressing the hypothesis space on generalization. Finally, their evaluation is computationally more friendly as it requires estimating MI between lower-dimensional variables. Specifically, our bounds depend on the alternative information theory measure called \emph{disintegrated mutual information} \citep[Definition 1.1]{negrea2019}. The disintegrated MI between $X$ and $Y$ given $U$ is defined as %
$\DMI{X}{Y}{U} = \KL{P_{X,Y | U}}{P_{X|U} \otimes P_{Y|U}}$, where $\textbf{KL}$ denotes the Kullback-Leibler divergence and $P_{X,Y | U}$ the conditional distribution of $(X,Y)$ given $U$, $P_{X|U}$ (respectively, $P_{Y|U}$) the conditional distribution of $X$ (resp., $Y$) given $U$.

\subsection{A first bound on $\generror{\mu}{\mathcal{A}^{(d)}}$} \label{sec:adapted_xu}

We first bound $\generror{\mu}{\mathcal{A}^{(d)}}$ by \emph{disintegrating} the proof of \Cref{thm:xu} \citep{xu2017}. 
\begin{theorem}\label{thm:adapted_xu}
    Assume for all $w' \in \R^d$ and $\Theta \in \sti(d,D)$, $\ell(\Theta w', Z)$ is $\sigma_\Theta$-sub-Gaussian under $Z \sim \mu$, where $\sigma_\Theta$ is a positive constant which may depend on $\Theta$. %
    Then,
	\begin{equation} \label{eq:bound_gen_err}
		| \generror{\mu}{\mathcal{A}^{(d)}} | \leq \sqrt{\frac{2}n} \,\E_{P_\Theta}\left[\sqrt{\sigma_\Theta^2\DMI{W'}{S_n}{\Theta}}\right] \,.
	\end{equation}
\end{theorem}

Note that since $W \!=\! \Theta W'$ implies $W' \!=\! \Theta^T W$, $\DMI{W'}{S_n}{\Theta}$ has significant parallels with the sliced mutual information \citep{goldfeld2022} with slicing in the first argument only, denoted $\mathsf{SI}^{(1)}_k(W; S_n)$. Sliced mutual information, however, is formulated with $W$ being independent of $\Theta$, which is not generally true in our setting (except in specific regimes such as the Gaussian mean estimation example below).   

\Cref{thm:adapted_xu} holds under a sub-Gaussianity assumption, which is slightly different than the one in \Cref{thm:xu}. It is immediate that the assumption in \Cref{thm:xu} implies the assumption in \Cref{thm:adapted_xu} with $\sigma_\Theta = \sigma$. For instance, consider the supervised learning setting, where for all $w = \Theta w' \in \mathrm{W}_{\Theta, d}$, $\ell(\Theta w', z) = \1_{f_{\Theta w'}(x) \neq y} \leq 1$. Then, by Hoeffding's lemma, \eqref{eq:application_xu} and \eqref{eq:bound_gen_err} both hold, with $\sigma_\Theta = \sigma = 2$. Conversely, if the assumption in \Cref{thm:adapted_xu} holds, then the assumption in \Cref{thm:xu} is met with $\sigma = \sup_{\Theta \in \sti(d, D)} \sigma_\Theta$.

Our bound entails notable advantages over \citet{xu2017}. First, \eqref{eq:bound_gen_err} is tighter than \eqref{eq:application_xu}, since $\E_{P_\Theta}\big[\sqrt{\DMI{W'}{S_n}{\Theta}} \big] \leq \sqrt{\MI{\Theta W'}{S_n}}$ (see \Cref{app:tightness_sn}). This is a natural consequence of the proof technique of \Cref{thm:adapted_xu}, which leverages disintegration, a strategy known to yield tighter characterizations of concave generalization bounds by Jensen's inequality \citep[§4.3]{hellstrom2023}.

Then, our bound is more tractable in regimes where \eqref{eq:application_xu} is fundamentally intractable, e.g., when $D$ is very large. Indeed, common estimators for $\DMI{W'}{S_n}{\Theta}$ exhibit faster convergence rates than $\MI{\Theta W'}{S_n}$ because $W'$ has a lower dimension than $\Theta W'$ ($d \ll D$). For instance, the theoretical guarantees of MINE \cite{belghazi2018mine} clearly support the favorable computational and statistical properties of our bounds. By \citep[Theorem 2]{goldfeld2022}, the approximation error induced by MINE decays rapidly as the dimensionality decreases. Additionally, some of the assumptions of MINE can be relaxed, allowing optimization over a larger class of distribution in lower-dimensional spaces \citep[Remark 6]{goldfeld2022}. One may argue that our bound requires computing the expectation of $\DMI{W'}{S_n}{\Theta}$ over $\Theta \sim P_\Theta$, which makes its evaluation expensive. However, in our experiments, we estimated this expectation with a Monte Carlo approximation and found that increasing the number of samples of $\Theta$ had little practical impact on our bounds. This is consistent with prior work \cite{li2018measuring}, which showed that test performance remains relatively stable across multiple values of $\Theta$ for a fixed $d$, while the choice of $d$ has a greater impact on the quality of solutions.

\subsection{A tighter bound via individual samples}

One limitation of \Cref{thm:adapted_xu} is that $\DMI{W'}{S_n}{\Theta} = +\infty$ if $P_{W', S_n | \Theta}$ is not absolutely continuous w.r.t $P_{W'|\Theta} \otimes \mu^{\otimes n}$, therefore the bound is vacuous. For instance, this happens when $W' = g(S_n)$ where $g$ is a smooth, non-constant deterministic function that may depend on $\Theta$. %
To overcome this issue, we combine disintegration with the \emph{individual-sample} technique introduced by \citet{bu2019}. This allows us to construct a bound in terms of the \emph{individual-sample disintegrated MI}, $\DMI{W'}{Z_i}{\Theta}$.%

\begin{theorem} \label{thm:bounded_loss} %
    Assume that \emph{(i)} for all $w' \in \R^d$ and $\Theta \in \sti(d,D)$, $\ell(\Theta w', Z)$ is $\sigma_\Theta$-sub-Gaussian under $Z \sim \mu$, where $\sigma_\Theta$ is a positive constant which may depend on $\Theta$; or \emph{(ii)} for all $\Theta \in \sti(d, D)$, $\ell(\Theta \tilde{W'}, \tilde{Z})$ is $\sigma_\Theta$-sub-Gaussian under $(\tilde{W'}, \tilde{Z}) \sim P_{W'|\Theta} \otimes \mu$. Then,
    \begin{align}
    \label{eq:genbound_boundloss}
        | \generror{\mu}{\mathcal{A}^{(d)}} | \leq \frac{1}{n} \sum_{i=1}^n \E_{P_\Theta}\left[\sqrt{2\sigma_\Theta^2 \DMI{W'}{Z_i}{\Theta}}\right] \,.
    \end{align}
\end{theorem}

Assumption \emph{(i)} is not stronger than \emph{(ii)} (e.g., one can adapt \citep[Remark 1]{bu2019}), and conversely (e.g., see Gaussian mean estimation in \Cref{subsec:connection_smi}). Note that \Cref{thm:bounded_loss} under assumption \emph{(ii)} is a particular case of a more general theorem, which we present in \Cref{app:genresult} for readability purposes. A key advantage of \Cref{thm:genbound} is its broader applicability as compared to \Cref{thm:adapted_xu,thm:bounded_loss}. We will illustrate this in \Cref{subsec:connection_smi} on linear regression, where the loss is not sub-Gaussian. 

Thanks to the individual-sample technique, our bound in \eqref{eq:genbound_boundloss} is no worse than the one in \Cref{thm:adapted_xu}. Indeed, assumption \emph{(i)} in \Cref{thm:bounded_loss} is the same as the one in \Cref{thm:adapted_xu}, and $\frac{1}{n} \sum_{i=1}^n \E_{P_\Theta}\left[\sqrt{2\sigma_\Theta^2 \DMI{W'}{Z_i}{\Theta}}\right] \leq \sqrt{\frac{2}{n}} \E_{P_\Theta} \left[ \sqrt{\sigma_\Theta^2 \DMI{W'}{S_n}{\Theta}} \right]$ (\Cref{prop:tightness_zi_sn}). In particular, if $W'$ is a deterministic function of $S_n$, the bound in \eqref{eq:genbound_boundloss} can be non-vacuous as opposed to \eqref{eq:bound_gen_err}.

Thanks to disintegration, the bound in \Cref{thm:bounded_loss} is no worse than the one in \Cref{thm:bu}, \ie, $| \generror{\mu}{\mathcal{A}^{(d)}} | \leq \frac1n \sum_{i=1}^n \sqrt{2 \sigma^2 \MI{\Theta W'}{Z_i}}$. This is justified using similar arguments as in \Cref{sec:adapted_xu} to compare the sub-Gaussian conditions and MI terms: see \Cref{prop:tightness_zi}. Interestingly, we show in \Cref{subsec:connection_smi} that \Cref{thm:bu} cannot be applied in a simple learning problem, because its assumptions do not trivially hold. In contrast, assumption \emph{(ii)} of \Cref{thm:adapted_xu} is easily verified because we condition on $\Theta$. This illustrates that, in addition to providing tighter bounds, disintegration helps formulate alternative sub-Gaussianity assumptions that can be milder.

\subsection{Applications} \label{subsec:connection_smi}

We illustrate more concretely the benefits of our findings over \citet{xu2017} and \citet{bu2019} in terms of tightness and applicability. Moreover, we uncover an interesting connection with the Sliced MI evaluated on $W = \mathcal{A}(S_n)$ and $Z_i$ where only $W$ is projected, \ie, $\kSMI{d}{W}{Z_i} = \E_{P_\Theta}\left[ \DMI{\Theta^\top W}{Z_i}{\Theta} \right]$. %

\textbf{Countable hypothesis space.} Our generalization bounds provide a clear explanation on why algorithms with low intrinsic dimension $d$ are likely to generalize well. Indeed, suppose that for any $\Theta \in \sti(d,D)$ and $w = \Theta w' \in \mathrm{W}_{\Theta, d}$, we have $\|w'\| \leq b_\Theta$, where $\|\cdot\|$ is the Euclidean norm. Then, using the same argumentation as in \citep[{\S}4.1]{xu2017} in \Cref{thm:adapted_xu},
\begin{equation} \label{eq:bound_gen_err_countable}
    | \generror{\mu}{\mathcal{A}^{(d)}} | \leq \sqrt{\frac{2d}{n}} \E_{P_\Theta} \left[ \sqrt{\sigma_\Theta^2 \log(2 b_\Theta \sqrt{dn})} \right] \,.
\end{equation}

We make two key observations based on \eqref{eq:bound_gen_err_countable}. First, the right-hand side (RHS) term decreases as $d$ shrinks. This confirms that algorithms trained on lower-dimensional random subspace tend to generalize better, as observed in practice \citep{li2018measuring}. Then, our bound can help guide architecture choices for practitioners who wish to control the generalization error, thanks to the explicit dependence on $d$, $n$ and $b_\Theta$. Note that given $\Theta$, achieving $\| w' \| \leq b_\Theta$ can easily be achieved in practice by quantizing $w'$.

\textbf{Gaussian mean estimation.} 
We now study the following problem inspired by \citep[Section IV.A]{bu2019}. The training dataset $S_n = (Z_1, \dots, Z_n)$ consists of $n$ i.i.d.\ samples from $\mathcal{N}(\mathbf{0}_D, \mI_D)$, where $\mathbf{0}_D \in \R^D$ is the zero vector. The objective function is $\riskn{w} \triangleq \frac1n \sum_{i=1}^n \| w - Z_i\|^2$. Consider the models $\mathcal{A}$ and $\mathcal{A}^{(d)}$ which minimize $\riskn{w}$ on $\setW = \R^D$ and $\setW_{\Theta, d}$ respectively. Then, $\mathcal{A}(S_n) = \bar{Z}$ and $\mathcal{A}^{(d)}(S_n) = \Theta \Theta^\top \bar{Z}$, where $\bar{Z} \triangleq \frac1n \sum_{i=1}^n Z_i$. Since $W'$ is a deterministic function of $S_n$ given $\Theta$, applying \Cref{thm:adapted_xu} would give a vacuous bound. Instead, we apply \Cref{thm:bounded_loss} and obtain
\begin{align}
    \generror{\mu}{\mathcal{A}^{(d)}} &\leq C_{D,d,n} \sum_{i=1}^n \E_{P_\Theta}\left[\sqrt{\DMI{\Theta^\top \bar{Z}}{Z_i}{\Theta}}\right] \label{eq:mean_estimation_bound_smi} \,, 
\end{align}
with $C_{D, d, n} \triangleq \frac2n \sqrt{d\left(1+\frac{1}{n}\right)^2 + (D-d)}$.  The detailed derivations are in \Cref{app:gmi}. For a fixed pair $(D,n)$, $C_{D, d, n}$ increases as $d$ goes to $D$. By adapting the proof of \citep[Proposition 2.2]{goldfeld2022}, one can show that $\E_{P_\Theta}\big[\sqrt{\DMI{\Theta^\top \bar{Z}}{Z_i}{\Theta}}\big]$ decreases as $d \to 0$, and by the data-processing inequality, for any $d \leq D$ and $\Theta \sim P_{\Theta}$, $\DMI{\Theta^\top \bar{Z}}{Z_i}{\Theta} \leq \MI{\bar{Z}}{Z_i}$. Therefore, the RHS term in \eqref{eq:mean_estimation_bound_smi} accurately captures that compressing the hypothesis space improves generalization. Note that here, $\DMI{\Theta^\top \bar{Z}}{Z_i}{\Theta}$ can actually be computed in closed form: since $Z_i$ and $\Theta^T \bar{Z}$ (given $\Theta$) are Gaussian random variables, and $\Theta^\top \Theta = \mI_d$, we have $\DMI{\Theta^\top \bar{Z}}{Z_i}{\Theta} = \frac{d}{2} \log(\frac{n}{n-1})$. We also show that the bound is sub-optimal, since it scales in $\mathcal{O}(1/\sqrt{n})$ as $n \to +\infty$, and $\generror{\mu}{\mathcal{A}^{(d)}} = 2d/n$.%

\begin{figure}
  \centering
    \includegraphics[width=.42\textwidth]{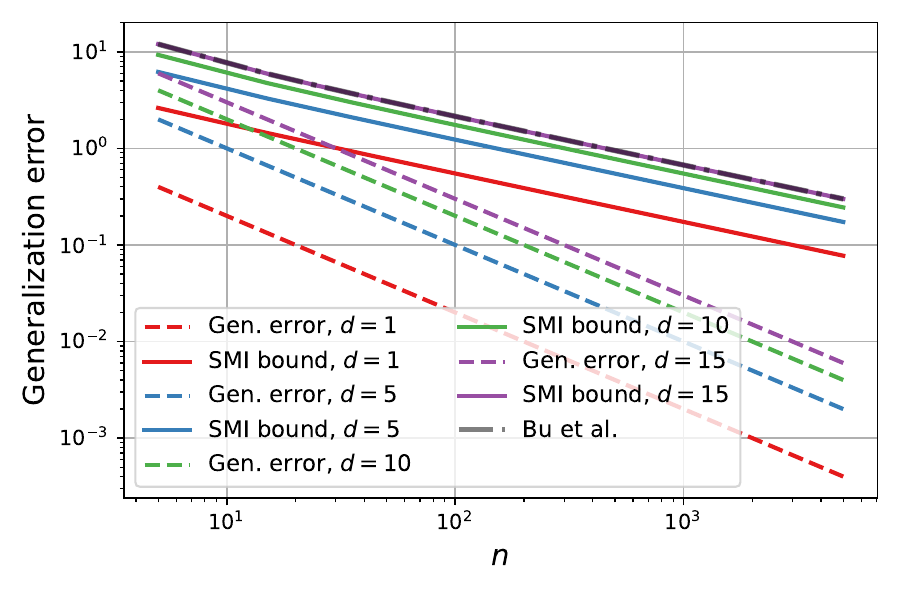}
    \vspace{-3mm}
    \caption{Gaussian mean estimation: generalization error and bound against $n$, for $D=15$, $d \in \{1, 5, 10, 15\}$. Errors and bounds decrease as $d \to 1$. The bound in \citet{bu2019} can only be applied for $d=D$. The scale is log-log.}
    \label{fig:gauss_mean}
    \vspace{-2mm}
\end{figure}
This example shows that our findings allow us to derive generalization bounds where prior work does not apply. Indeed, for $d = D$, our bound boils down to the one derived by \citet{bu2019}; but when $d < D$, their strategy cannot be applied: the distribution of $\ell(\Theta W', Z) = || \Theta W' - Z ||^2$ is unknown, making it highly non-trivial to verify the sub-Gaussian condition by \citet{bu2019} (in particular, $\Theta W'$ is not Gaussian). We overcome this issue via disintegration, \ie, by conditioning on $\Theta$: we prove in \Cref{app:gmi} that $\ell(\Theta \tilde{W'}, \tilde{Z}) = \| \Theta \tilde{W'} - \tilde{Z} \|^2$ is sub-Gaussian \emph{given} $\Theta$ (\ie, under $(\tilde{W}', \tilde{Z}) \sim P_{W' | \Theta} \otimes \mu$), which allows the application of \Cref{thm:bounded_loss}. 
Finally, by Jensen's inequality and $\bar{Z} = W$, $\E_{P_\Theta}\left[\sqrt{\DMI{\Theta^\top \bar{Z}}{Z_i}{\Theta}}\right] \leq \sqrt{\kSMI{d}{W}{Z_i}}$, thus our bound is controlled by SMI. We report $\generror{\mu}{\mathcal{A}^{(d)}}$ and this bound in \Cref{fig:gauss_mean}. In \Cref{app:linreg}, we apply our theorem on linear regression and obtain a bound that also relies on SMI.

\section{Information-Theoretic Generalization Bounds for Compressible Models} \label{subsec:ratedis}

The above bounds require the learned weights $W$ to lie in $\mathrm{W}_{\Theta, d}$. When $d$ is very small, this constraint can be restrictive and significantly deteriorate the performance of the model, as we illustrate in \Cref{sec:exp}. That said, since our MI-based bounds generally increase with increasing $d$, it is important to keep $d$ small. Motivated by recent work applying rate-distortion theory to input-output MI generalization bounds \citep{sefidgaran2022rate}, we establish the following result for \emph{approximately compressible} weights and Lipschitz loss.

\begin{theorem}\label{thm:rateDis}
    Consider $\mathcal{A} : \mathrm{Z}^n \to \mathrm{W}$ s.t. $\mathcal{A}$ may take $\Theta \sim P_\Theta$ into account to output $W$. Assume there exists $C > 0$ s.t. $\ell(\tilde{W}, \tilde{Z}) \leq C$ almost surely. Assume for any $z \in \setZ$, $\ell(\cdot, z) : \setW \to \R_+$ is $L$-Lipschitz, \ie, $\forall (w_1, w_2) \in \setW \times \setW, |\ell(w_1, z) - \ell(w_2, z)| \leq L \rho(w_1, w_2)$, where $\rho$ is a metric on $\setW$. Then,
    \begin{align} 
        \nonumber | \generror{\mu}{\mathcal{A}} | \leq \; &2L\E_{P_{W | \Theta} \otimes P_\Theta} \big[ \rho(W, \Theta \Theta^\top W) \big] \\
        &+ \frac{C}{n} \sum_{i=1}^n \E_{P_\Theta}\left[\sqrt{\frac{\DMI{\Theta^\top W}{Z_i}{\Theta}}{2}}\,\right] \,.\label{eq:ratedisbound}
    \end{align}
    \end{theorem}
    This result shows a trade-off between distortion and generalization, aligning with prior work on generalization through the rate-distortion theory (e.g., \citet{bu2020}, \citet{sefidgaran2022rate}; see \Cref{sec:relatedwork} for a detailed comparison). In the limit case where $d = D$, we retrieve the bound by \citet{xu2017}. 

    The proof of \Cref{thm:rateDis} consists in considering two models $\mathcal{A} : \mathrm{Z}^n \to \R^D$ and $\mathcal{A}' : \mathrm{Z}^n \to \mathrm{W}_{\Theta, d}$ such that $\mathcal{A}(S_n) = W$ may depend on $\Theta \sim P_\Theta$, and $\mathcal{A}'(S_n) = \Theta \Theta^\top W$. We then use the triangle inequality to obtain $| \generror{\mu}{\mathcal{A}} | \leq | \generror{\mu}{\mathcal{A}} - \generror{\mu}{\mathcal{A}'} | + | \generror{\mu}{\mathcal{A}'} | $. Finally, we bound the first term (the \emph{distortion} term) using the Lipschitz condition and the second term (the \emph{rate} term) using \Cref{thm:bounded_loss}.

    Using similar arguments and \Cref{thm:adapted_xu}, we derive a second rate-distortion bound based on quantization, which does not require estimating MI. %
    
    \begin{theorem} \label{thm:quant_ratedis}
    Assume the conditions of \Cref{thm:rateDis} hold.
    Furthermore, suppose that $\|\Theta^\top W \| \leq M$ for $(W, \Theta) \sim P_{W|\Theta} \otimes P_\Theta$. %
    Consider a function $\mathcal{Q}$ quantizing $\Theta^\top W$ such that $\rho\big(\Theta^\top W, \mathcal{Q}(\Theta^\top W)\big) \leq \delta$. Then,
     \begin{align}
        \nonumber | \generror{\mu}{\mathcal{A}} | \leq&\; 2L\E_{P_{W|\Theta} \otimes P_\Theta} \big[\rho\big(W,\Theta \mathcal{Q}(\Theta^\top W)\big)\big] \nonumber \\
        &+ C  \E_{P_\Theta}\left[\sqrt{\frac{\DMI{\mathcal{Q}(\Theta^\top W)}{S_n}{\Theta}}{2n}}\right] \, \label{eq:ratedisbound_2} \\
        \leq&\; 2L\left(\E_{P_{W|\Theta} \otimes P_\Theta}\big[\rho(W, \Theta \Theta^\top W) \big] + \delta \right) \label{eq:quant_ratedis_0} \\ 
        &+ C\sqrt{\frac{d \log (2M\sqrt{d}/\delta)}{2n}} \label{eq:quant_ratedis} \,.
    \end{align}
    \end{theorem}
    
    Note that $\|\Theta^\top W \| \leq M$ is a mild assumption, since in general, this is a result of enforcing Lipschitz continuity (e.g., the Lipschitz neural networks studied by \citet{bethune2023} require weights with bounded norms). We will set $\delta = 1/\sqrt{n}$ to reflect the fact that training on more samples reduces the generalization error.

    The MI term in \eqref{eq:ratedisbound} or \eqref{eq:ratedisbound_2} is evaluated between the training data $S_n$ and a low-dimensional, potentially quantized projection of $W$. This makes our rate-distortion bounds simpler to estimate than standard information-theoretic ones that rely on $\MI{S_n}{W}$. Our bound in \eqref{eq:quant_ratedis_0}-\eqref{eq:quant_ratedis} further reduces the computational complexity by bounding the MI term in \eqref{eq:ratedisbound_2} using similar arguments to those in \citep[§4.1]{xu2017}. It should be viewed as an interpretable and easily computable alternative bound that supports our main message: almost-compressibility on random subspaces implies better generalization. Indeed, given that the term in \eqref{eq:quant_ratedis} increases with increasing $d$, a tighter generalization bound can be achieved for $d < D$, provided that the corresponding distortion in \eqref{eq:quant_ratedis_0} (which measures the rate of compressibility) is sufficiently small. Practitioners also do not need to quantize the weights to evaluate \eqref{eq:quant_ratedis_0}-\eqref{eq:quant_ratedis}, making the process computationally more efficient: assuming the existence of a quantizer $\mathcal{Q}$ as described in \Cref{thm:quant_ratedis} is sufficient.

    Our theoretical findings provide concrete guidelines on how to tighten the generalization error bounds in practice. First, the value of the Lipschitz constant $L$ can be directly controlled through the design of the neural network, as we explain in \Cref{sec:exp} and \Cref{app:rate_dis_fc}. The term $\E_{P_{W|\Theta} \otimes P_\Theta} \big[ \rho(W, \Theta \Theta^\top W) \big]$ can be regularized by simply adding a penalty term $\lambda \E_{P_{W|\Theta} \otimes P_\Theta} \big[ \rho(W, \Theta \Theta^\top W) \big]$ to the training objective. Depending on the value of the hyperparameter $\lambda$, this regularization can encourage solutions to be close to the subspace $\mathrm{W}_{\Theta, d}$, \ie, having low distortion from the compressed weights. The choice of $d$ is also important and can be tuned to balance the MI term with the distortion required (how small $\lambda$ needs to be) to achieve low training error. Indeed, choosing a higher $\lambda$ increases the importance of the regularization term, effectively reducing the importance of the empirical risk. Hence, the empirical risk may rise, which in most cases will increase the training error.

\section{Empirical Analysis} \label{sec:exp}

\begin{figure}[t]
    \centering    
    \includegraphics[width=.42\textwidth]{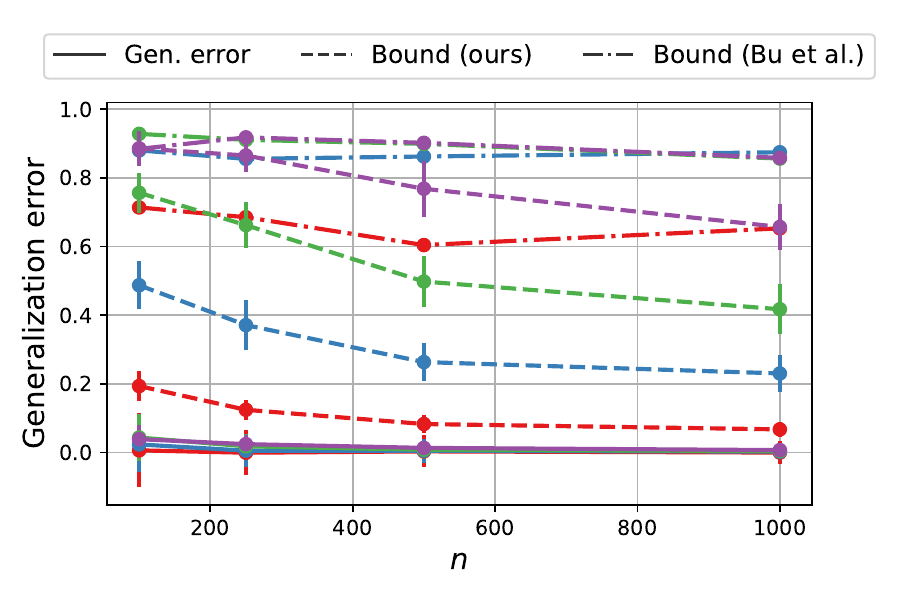}
    \includegraphics[width=.42\textwidth]{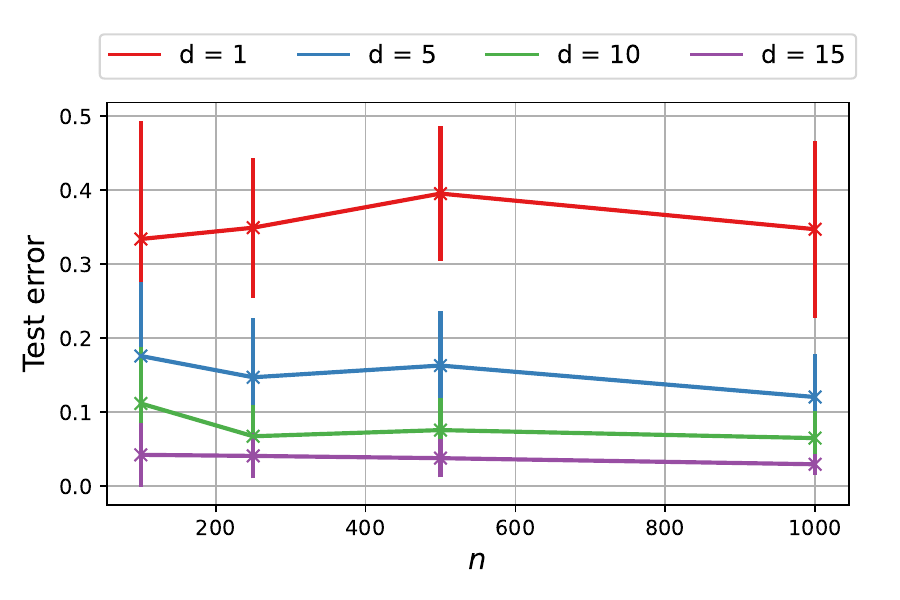}
    \vspace{-3mm}
    \caption{Illustration of our bound \eqref{eq:genbound_boundloss} and \cite{bu2019} on binary classification of Gaussian data of dimension $20$ with logistic regression trained on $\mathrm{W}_{\Theta, d}$}
    \label{fig:logreg}
    \vspace{-4mm}
\end{figure}

To illustrate our findings and their practical impact, we train several neural networks for classification, and evaluate their generalization error and our bounds. This requires compressing NNs (via random slicing and quantization) and estimating MI. We explain our methodology below, and refer to \Cref{app:methodology} for more details and results. We provide the code to reproduce the experiments\footnote{{\small Code is available here: \url{https://github.com/kimiandj/slicing_mi_generalization}}}.

\textbf{Random projections.} To sample $\Theta \in \R^{D \times d}$ such that $\Theta^\top \Theta = \mI_d$, we construct an orthonormal basis using the singular value decomposition of a random matrix $\Gamma \in \R^{D \times d}$ whose entries are i.i.d.\ from $\mathcal{N}(0, 1)$. Since the produced matrix $\Theta$ is dense, the projection $\Theta^\top w$ induces a runtime of $\mathcal{O}(dD)$. To improve scalability, we use the sparse projector by \citet{li2018measuring} and the Kronecker product projector by \citet{lotfi2022pac}, which compute $\Theta^\top w$ in $\mathcal{O}(d\sqrt{D})$ and $\mathcal{O}(\sqrt{dD})$ operations respectively, and require storing only $\mathcal{O}(d\sqrt{D})$ and $\mathcal{O}(\sqrt{dD})$ matrix elements respectively.

\textbf{Quantization.} We use the quantizer by \citet{lotfi2022pac}, which simultaneously learns quantized weights $W'$ and quantized levels $(c_1, \cdots, c_L)$. This allows us to highly compress NNs and bypass the estimation of MI: for any $\Theta \sim P_\Theta$, $\DMI{W'}{S_n}{\Theta} \leq H^\Theta(W') \leq \lceil d \times H(p) \rceil + L \times (16 + \lceil \log_2 d \rceil) + 2$, where $H^\Theta(W')$ denotes the entropy of $W'$ given $\Theta$, and $H(p) \triangleq -\sum_{l=1}^L p_l \log(p_l)$ is the entropy of the quantized level ($p_l$ is the empirical probability of $c_l$). %

\textbf{Estimating MI.} In our practical settings, the MI terms arising in the generalization bounds cannot be computed exactly, so we resort to two popular estimators: the $k$-nearest neighbor estimator ($k$NN-MI, \citealp{kraskov2004}) and MINE \citep{belghazi2018mine}. We obtain NaN values with $k$NN-MI for $d > 2$ thus only report the bounds estimated with MINE. In our experiments, the use of MINE was not a practical issue because $d$ had low to relatively high values. %

\begin{figure}
  \centering
  \begin{subfigure}[b]{.5\textwidth}
    \centering
    \includegraphics[width=.4\textwidth]{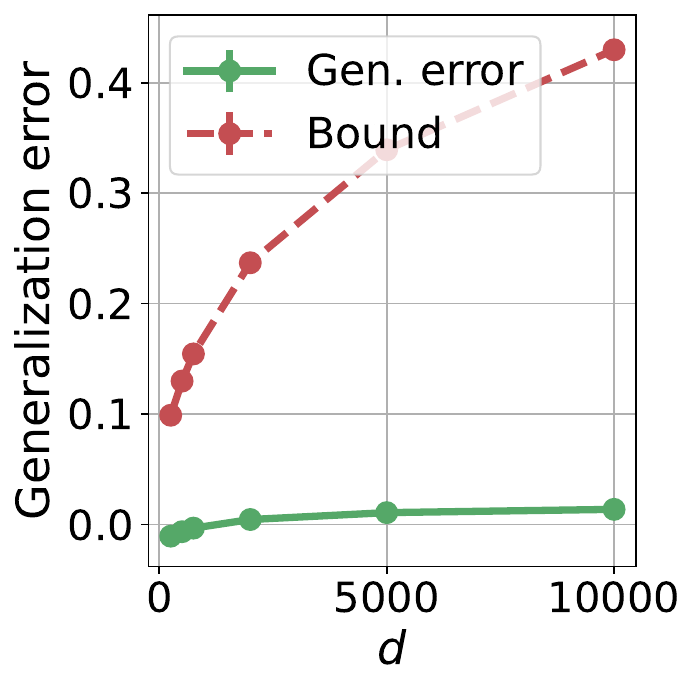}
    \includegraphics[width=.4\textwidth]{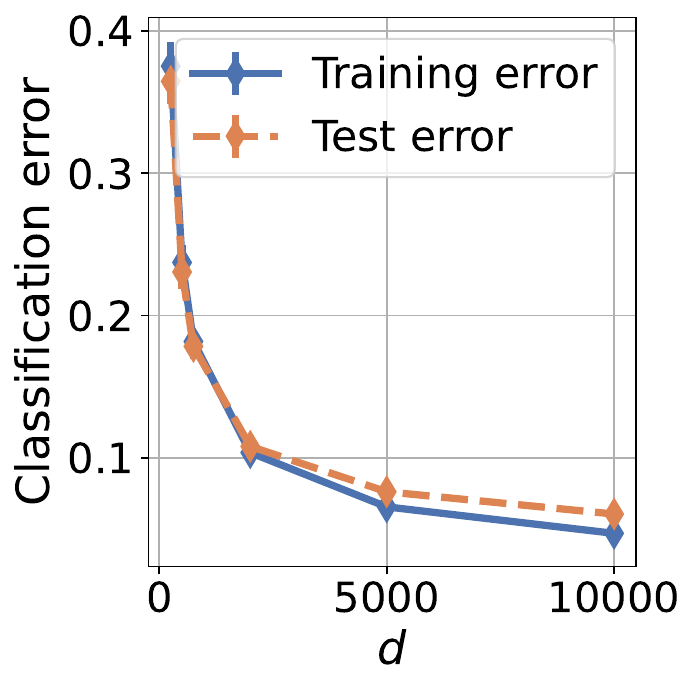}
    \vspace{-2mm}
    \caption{MNIST}
    \vspace{1mm}
  \label{fig:quant_mnist}
  \end{subfigure}
  \begin{subfigure}[b]{.5\textwidth}
    \centering
    \includegraphics[width=.4\textwidth]{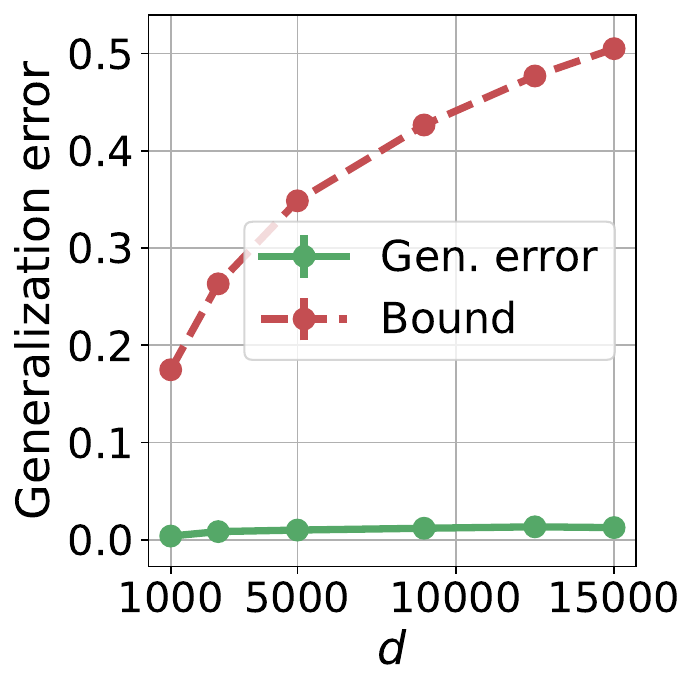}
    \includegraphics[width=.4\textwidth]{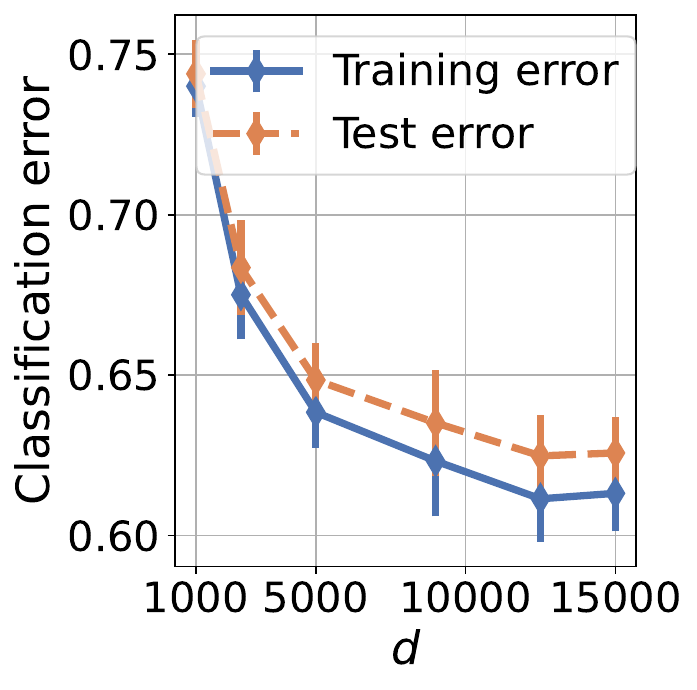}
    \vspace{-1mm}
    \caption{CIFAR-10}
  \label{fig:quant_cifar10}
  \end{subfigure}
  \vspace{-4mm}
  \caption{Generalization bounds with NNs for image classification. The weights are projected and quantized. %
  }
  \vspace{-6mm}
  \label{fig:bounds_nns}
\end{figure}
\subsection{Generalization bounds for models trained on $\mathrm{W}_{\Theta, d}$} \label{subsec:illustration_compress_bounds}

\textbf{Binary classification with logistic regression.} %
We consider the same setting as \citet[{\S}VI]{bu2019}: each data point $Z = (X, Y)$ consist of features $X \in \R^s$ and labels $Y \in \{0, 1\}$, $Y$ is uniformly distributed in $\{0, 1\}$, and $X | Y \sim \mathcal{N}(\mu_Y, 4 \mI_s)$ with $\mu_0 = (-1, \dots, -1)$ and $\mu_1 = (1, \dots, 1)$. We use a linear classifier and evaluate the generalization error based on the loss function $\ell(w, z) = \1_{\hat{y} \neq y}$, where $\hat{y}$ is the prediction of input $x$ defined as $\hat{y} \triangleq \1_{\bar{w}^T x + w_0 \geq 0}$, $\forall w = (\bar{w}, w_0) \in \R^{s+1}$. We train a logistic regression on $\mathrm{W}_{\Theta, d}$ and estimate the generalization error. Since $\ell$ is bounded by $C = 1$, we approximate the generalization error bound from \Cref{thm:bounded_loss} for $d < D$, and \citet[Prop. 1]{bu2019} for $d = D$. %
\Cref{fig:logreg} reports the results for $s = 20$ and different values of $n$ and $d$: we observe that our bound holds and accurately reflects the behavior of the generalization error against $(n, d)$. Our methodology also provides tighter bounds than \citet{bu2019}, and the difference increases with decreasing $d$.
On the other hand, the lower $d$, the lower generalization error and its bound, but the higher the test risk (\Cref{fig:logreg}). This is consistent with prior empirical studies \cite{li2018measuring} and explained by the fact that lower values of $d$ induce a more restrictive hypothesis space, thus make the model less expressive.

\textbf{Multiclass classification with NNs.} Next, we evaluate our generalization error bounds for neural networks trained on image classification. Denote by $f(w,x) \in \R^K$ the output of the NN parameterized by $w$ given an input image $x$, with $K > 1$ the number of classes. The loss is $\ell(w, z) = \1_{\hat{y} \neq y}$, with $\hat{y} = \argmax_{i \in \{1, \dots, K\}} [f(w,x)]_i$. We train fully-connected NNs to classify MNIST and CIFAR-10 datasets, with $D = 199\,210$ and $D = 656\,810$ respectively: implementation details are given in \Cref{appendix:NNs}. Even though we significantly decrease the dimension of the parameters by slicing, $d$ can still be quite high thus obtaining an accurate estimation of $\DMI{W'}{S_n}{\Theta}$ remains costly. To mitigate this issue, in addition to slicing, we discretize $W'$ with the quantizer by \citet{lotfi2022pac} and evaluate \Cref{thm:adapted_xu} with $\DMI{W'}{S_n}{\Theta}$ replaced by $\lceil d \times H(p) \rceil + L \times (16 + \lceil \log_2 d \rceil) + 2$, as discussed at the beginning of \Cref{sec:exp}. Our results in \Cref{fig:bounds_nns} illustrate that employing slicing followed by quantization enables the computation of non-vacuous generalization bounds for NNs, while still maintaining test performance for adequate values of $d$ (which is consistent with \citet{li2018measuring}). We point out that in these high-dimensional problems, slicing is unequivocally the key to making information-theoretic bounds possible to estimate: quantization alone is far from sufficient in this regime. For example, even binary quantization of $D$ weights would yield $2^D$ states, requiring an unimaginably large number of samples to accurately estimate the mutual information term. Additional results on MNIST and Iris datasets are given in \Cref{appendix:NNs}.

\subsection{Rate-distortion bounds} \label{subsec:ratedis_exp}

\begin{figure}
  \centering
  \begin{subfigure}[b]{0.25\textwidth}
    \centering
    \includegraphics[width=\textwidth]{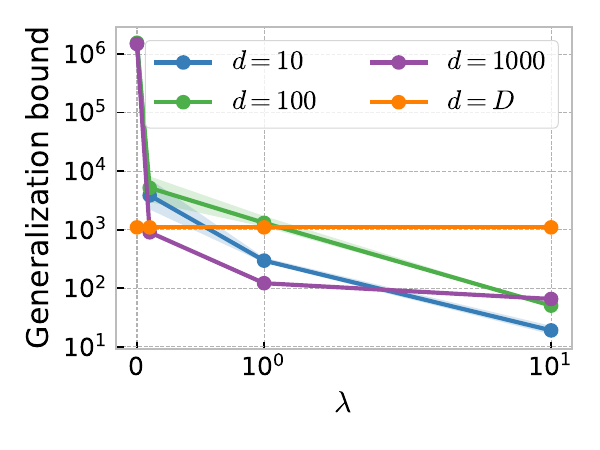}
    \vspace{-7mm}
    \caption{Rate-distortion bound}
    \label{fig:ratedis_bound}
  \end{subfigure}%
  \begin{subfigure}[b]{0.25\textwidth}
    \centering
    \includegraphics[width=\textwidth]{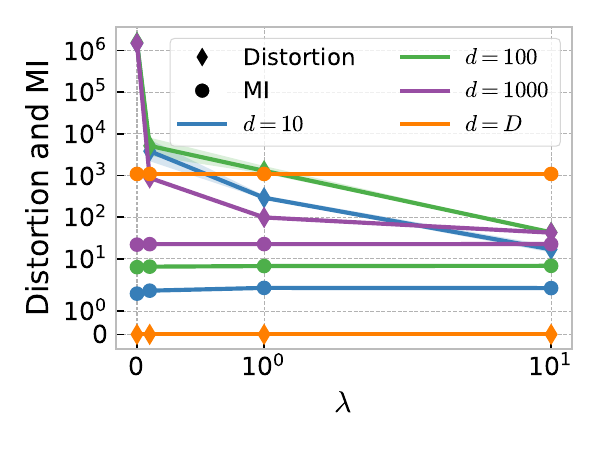}
    \vspace{-7mm}
    \caption{Distortion and MI terms}
    \label{fig:dist_mi}
  \end{subfigure}
  \begin{subfigure}[b]{0.25\textwidth}
    \centering
    \includegraphics[width=\textwidth]{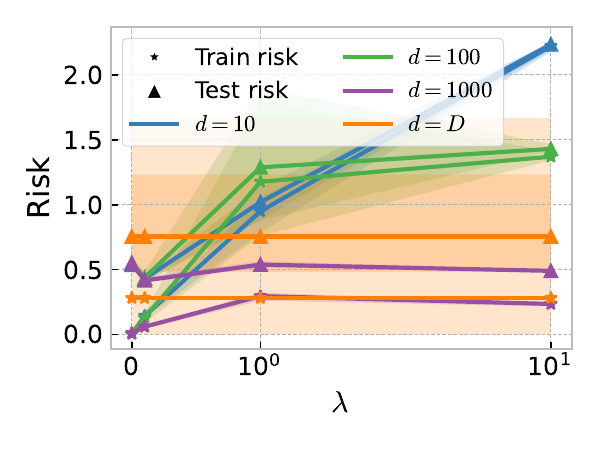}
    \vspace{-7mm}
    \caption{Train and test risks}
    \label{fig:ratedis_generr}
  \end{subfigure}%
  \begin{subfigure}[b]{0.25\textwidth}
    \centering
    \includegraphics[width=\textwidth]{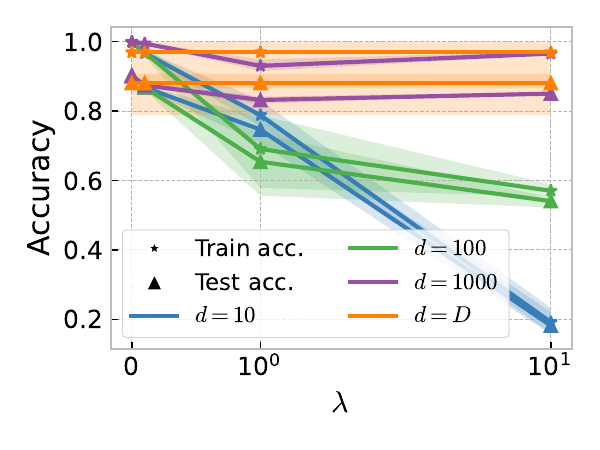}
    \vspace{-7mm}
    \caption{Train and test accuracies}
    \label{fig:ratedis_accs}
  \end{subfigure}
  \caption{Generalization errors and rate-distortion bounds for feedforward NNs trained on MNIST. Results are averaged over 5 runs. Shaded areas represent the 2.5\% and 97.5\% percentiles. For each run, expectations are computed with Monte Carlo estimates over 5 samples of $\Theta$.}
  \vspace{-3mm}
\end{figure}

We evaluate our rate-distortion generalization bounds for neural networks trained on image classification. For $q \in \mathbb{N}^*$, let $f : \mathrm{W} \times \mathrm{X} \to \R^{K}$ be a $q$-layer feedforward network with ReLU activations. We train $f$ using a slightly modified cross-entropy loss defined for $w \in \mathrm{W}$ and $z = (x, y) \in \mathrm{X} \times \{1, \dots, K\}$ as $\ell(w, z) = - \log(\hat{p}(w,x)_y)$, where $\hat{p}(w,x) = \max(p(w,x), p_{\min})$, $p(w,x) = e^{f(w,x)} / \1^\top e^{f(w,x)} \in (0, 1]^K$ and $p_{\min} > 0$. This loss is bounded from above by $-\log(p_{\min})$ and for any $z$, $\ell(\cdot, z)$ is Lipschitz-continuous with constant $\sqrt{2}$. Assuming that the weight matrix of each layer has bounded spectral norm ($\forall i \in \{1, \dots, q\}$, $\| W^{(i)} \|_2 \leq M$), we show that \Cref{thm:rateDis} or \ref{thm:quant_ratedis} applies with the distortion $\rho(W,\Theta \Theta^\top W) = \sum_{i=1}^q \| W^{(i)} - (\Theta \Theta^\top W)^{(i)} \|_2 $ and the constants $C$ and $L$ specified in \Cref{thm:rate_dis_bound_fc}. 

We train a $3$-layer feedforward NN $f$ to classify MNIST. The first two layers each contain $1000$ neurons and the final layer has $K = 10$ neurons, thus the total number of parameters is $D = 1000 \cdot (784 + 1000 + 10) = 1\,794\,000$. Our goal is to evaluate the generalization error and its rate-distortion bound in \Cref{thm:quant_ratedis} for different values of subspace dimension $d$ and regularization coefficient $\lambda$. To this end, we parameterize $f$ with $w = \Theta (\Theta^\top w_1) + \bar{\Theta} (\bar{\Theta}^\top w_2) \in \R^D$, where $\Theta \in \sti(d,D)$ is randomly generated at initialization and $\bar{\Theta} \in \R^{D \times (D-d)}$ is such that $[\Theta, \bar{\Theta}] \in \R^{D \times D}$ forms an orthogonal basis of $\R^D$. At each run, we train $f$ on a random subset of MNIST with $n = 1000$ samples for 5 different samples of $\Theta$. We set $p_{\text{min}} = 1\text{e-}4$. To estimate the generalization error, we approximate the population risk on a test dataset of $10\,000$ samples. Our results clearly depict the interplay between $\lambda$ and $d$ and their impact on the generalization error, bound and risk. Specifically, a higher $\lambda$ makes our model more compressible by encouraging its parameters to lie on the $d$-dimensional subspace characterized by $\Theta$ (see our discussion in \Cref{subsec:ratedis}). This has two main consequences, as predicted by our theory and illustrated in our plots. First, a higher $\lambda$ leads to a lower generalization error (\Cref{fig:ratedis_generr}) and a tighter rate-distortion bound, the distortion term being smaller (\Cref{fig:ratedis_bound,fig:dist_mi}). Second, as $\lambda$ increases and $d$ decreases, the train/test risk become higher (\Cref{fig:ratedis_generr}). This is consistent with \citet{li2018measuring}, as this regime effectively reduces to training on a low-dimensional random subspace. To further demonstrate the trade-off between high compressibility/low generalization error and high train/test error, we also plot the distortion (\ie, $2L(\mathbb{E}[\rho(W,\Theta \Theta^T W)] + 1/\sqrt{n})$) and MI term ($C\sqrt{d\log(2M\sqrt{dn})/(2n)}$) against $\lambda$, for different values of $d$ (\Cref{fig:dist_mi}). Additionally, we plot the test and train accuracies vs. $\lambda$ and $d$ (\Cref{fig:ratedis_accs}). We observe there exist combinations of $\lambda$ and $d$ that yield tighter generalization error bound while inducing satisfactory training and test errors, e.g., $(\lambda, d) = (10, 1000)$. This suggests that with carefully chosen $\lambda$ and $d$, our methodology can tighten generalization bounds while preserving model performance. This is a significant step towards practical relevance of information-theoretic bounds: to our knowledge, such bounds applied to neural networks have been fundamentally intractable and pessimistic, thus lacking practical use beyond very small toy examples (we refer to \Cref{sec:relatedwork} for an expanded discussion). In contrast, by taking into account almost-compressibility via random slicing and quantization, we can derive bounds that are much easier to compute, and develop a theoretically-grounded regularization scheme to effectively control the generalization error in practice.

\section{Conclusion}

In this work, we combined recent empirical compression methods for learning models, such as NNs, with generalization bounds based on input-output MI. Our results indicate that random slicing is a very interesting scheme, as it is easy to implement, performs well, and is highly suitable for practically computable and tighter information-theoretic bounds. We also explore a notion of approximate compressibility, \ie, rate-distortion, where the learned parameters are close to a quantization of the compressed subspace but do not lie on it exactly. This framework provides more flexibility, enabling the model to maintain good training error even with a smaller subspace dimension $d$, while ensuring that the resulting generalization bounds are as tight as possible and permitting the use of analytical bounds on the MI instead of difficult-to-compute MI estimates. %
Our rate-distortion approach also motivated a weight regularization technique to make trained NNs as approximately compressible as possible and ensure that our bound is small in practice. %
Future work includes a more detailed exploration of strategies for using our bounds to help inform selection and design of NN architectures, and exploring bounds and regularizers based on other successful compression methods. %

\section*{Acknowledgements}

Kimia Nadjahi acknowledges the generous support from the MIT Postdoctoral Fellowship Program for Engineering Excellence. The MIT Geometric Data Processing Group acknowledges the generous support of Army Research Office grants W911NF2010168 and W911NF2110293, from the CSAIL Systems that Learn program, from the MIT–IBM Watson AI Laboratory, from the Toyota–CSAIL Joint Research Center, and from an Amazon Research Award. Kimia Nadjahi expresses gratitude to the organizers of the ICML 2023 workshop ``Neural Compression: From Information Theory to Applications'', where an early version of this work was accepted for an oral presentation and sparked fruitful discussions with the attendees.

\section*{Impact Statement}
This paper presents work whose goal is to advance the field of Machine Learning. There are many potential societal consequences of our work, none which we feel must be specifically highlighted here.

\bibliography{main.bib}
\bibliographystyle{icml2024}

\flushcolsend
\newpage
\appendix
\onecolumn

\section{Postponed Proofs for \Cref{sec:bounds}}

{\textbf{Notation.} $\drisk{w'}{\Theta} = \E_{Z \sim \mu}[\ell^\Theta(w', Z)]$ and $\driskn{w'}{\Theta} \triangleq \frac1n \sum_{i=1}^n \ell^\Theta(w', z_i)$, $\forall w = \Theta w' \in \setW_{\Theta,d}$ and $\ell^\Theta(w', z) \triangleq \ell(\Theta w', z)$. The generalization error of $\mathcal{A}^{(d)}$ is $\generror{\mu}{\mathcal{A}^{(d)}} = \E[\drisk{W'}{\Theta} - \driskn{W'}{\Theta}]$ with the expectation taken over $P_{W' | \Theta, S_n} \otimes P_\Theta \otimes \mu^{\otimes n}$. }

\subsection{Proof of \Cref{thm:genbound}} \label{app:genresult}

Consider three random variables $X \in \mathrm{X}$, $Y \in \mathrm{Y}$ and $U \in \mathrm{U}$. Denote by $P_{X,Y,U}$ their joint distribution and by $P_X, P_Y, P_U$ the marginals. Let $\tilde{X}$ (respectively, $\tilde{Y}$) be an independent copy of $X$ (resp., $Y$) with joint distribution $P_{\tilde{X}, \tilde{Y}} = P_{\tilde{X}} \otimes P_{\tilde{Y}}$. Given $U$, let $f^U : \mathrm{X} \times \mathrm{Y} \to \R$ be a mapping parameterized by $U$, and denote by $K_{f^U(\tilde{X}, \tilde{Y})}$ the cumulant generating function of $f^U(\tilde{X}, \tilde{Y})$, \ie~for $t \in \R$,
\begin{equation}
    K_{f^U(\tilde{X}, \tilde{Y})}(t) = \log \E\big[e^{ t(f^U(\tilde{X}, \tilde{Y}) - \E[f^U(\tilde{X}, \tilde{Y})])} \big] 
\end{equation}
where the expectations are taken w.r.t. $P_{X|U} \otimes P_{Y|U}$.

\begin{lemma} \label{lem:adapted_bu}
    Suppose that for any $U \sim P_U$, there exists $b_+ \in \mathbb{R}_+^* \cup \{ +\infty\}$ and a convex function $\varphi_+(\cdot, U) : [0, b_+) \to \R$ such that $\varphi_+(0, U) = \varphi'_+(0, U) = 0$ and for $t \in [0, b_+)$, $K_{f^U(\tilde{X}, \tilde{Y})}(t) \leq \psi_+(t, U)$. Then, 
    \begin{equation}
        \E_{P_{X, Y, U}}[f^U(X, Y)] - \E_{P_{\tilde{X}, \tilde{Y}, U}}[f^U(\tilde{X}, \tilde{Y})] \leq \E_{P_U} \left[ \inf_{t \in [0, b_+)} \frac{\DMI{X}{Y}{U} + \psi_+(t, U)}{t} \right] \,. \label{eq:phi_plus_0}
    \end{equation}
    
    Suppose that for any $U \sim P_U$, there exists $b_- \in \R_+^* \cup \{ +\infty\}$ and a convex function $\varphi_-(\cdot, U) : [0, b_-) \to \R$ such that $\varphi_-(0, U) = \varphi'_-(0, U) = 0$ and for $t \in (b_-, 0]$, $K_{f^U(\tilde{X}, \tilde{Y})}(t) \leq \psi_-(-t, U)$. Then, 
    \begin{equation} \label{eq:phi_minus}
        \E_{P_{\tilde{X}, \tilde{Y}, U}}[f^U(\tilde{X}, \tilde{Y})] - \E_{P_{X, Y, U}}[f^U(X, Y)] \leq \E_{P_U} \left[ \inf_{t \in [0, -b_-)} \frac{\DMI{X}{Y}{U} + \psi_-(t, U)}{t} \right] \,.
    \end{equation}
\end{lemma}

\begin{proof}
    Let $U \sim P_U$. By Donsker-Varadhan variational representation,
    \begin{align}
        \DMI{X}{Y}{U} &= \KL{P_{(X, Y) | U}}{P_{X|U} \otimes P_{Y|U}} \\
        &= \sup_{g \in \mathcal{G}^U}\; \E_{P_{(X, Y) | U}}[g^U(X, Y)] - \log \E_{P_{X | U} \otimes P_{Y | U}}[e^{g^U(\tilde{X}, \tilde{Y})}]
    \end{align}
    where $\mathcal{G}^U \triangleq \{ g^U : \mathrm{X} \times \mathrm{Y} \to \R\;\text{s.t.}\; \E_{P_{X | U} \otimes P_{Y | U}}[e^{g^U(\tilde{X}, \tilde{Y})}] < \infty \}$. Therefore, for any $t \in [0, b_+)$,
    \begin{align}
        \mathbf{KL}(P_{(X, Y) | U} \| P_{X|U} \otimes P_{Y|U}) &\geq t \E[f^U(X, Y)] - \log \E[e^{t f^U(\tilde{X}, \tilde{Y})}] \\
        &\geq t \left( \E[f^U(X, Y)] - \E[f^U(\tilde{X}, \tilde{Y})] \right) - \psi_+(t, U) \label{eq:phi_plus}
    \end{align}
    where \eqref{eq:phi_plus} follows from assuming that for $t \in [0, b_+), K_{f^U(\tilde{X}, \tilde{Y})}(t) \leq \psi_+(t, U)$. Hence,
    \begin{align}
        \E[f^U(X, Y)] - \E[f^U(\tilde{X}, \tilde{Y})] &\leq \inf_{t \in [0, b_+)} \frac{\DMI{X}{Y}{U} + \psi_+(t, U)}{t} \,. \label{eq:phi_plus_2}
    \end{align}
    We obtain the final result \eqref{eq:phi_plus_0} by taking the expectation of \eqref{eq:phi_plus_2} over $P_U$.

    We can prove analogously that \eqref{eq:phi_minus} holds, assuming for $t \in [0, b_-), K_{f^U(\tilde{X}, \tilde{Y})}(t) \leq \psi_-(-t, U)$. 
    
\end{proof}

\begin{theorem} \label{thm:genbound}
    Assume that for $\Theta \sim P_\Theta$, there exists $C_- \in \R_+^* \cup \{ +\infty\}$ s.t. for $t \in (C_-, 0]$, $K_{\ell^\Theta(\tilde{W}', \tilde{Z})}(t) \leq \psi_-(-t, \Theta)$, where $\psi_-(\cdot, \Theta)$ is convex and $\psi_-(0, \Theta) = \psi'_-(0, \Theta) = 0$. Then, 
    \begin{align} 
        \generror{\mu}{\mathcal{A}^{(d)}} &\leq \frac1n \sum_{i=1}^n \E_{P_\Theta} \left[ \inf_{t \in [0, -C_-)} \frac{\DMI{W'}{Z_i}{\Theta} + \psi_-(t, \Theta)}{t} \right] \,. \label{eq:bound_gen_psi_minus}
    \end{align}
    Assume that for $\Theta \sim P_\Theta$, there exists $C_+ \in \R_+^* \cup \{ +\infty\}$ s.t. for $t \in [0, C_+)$, $K_{\ell^\Theta(\tilde{W}', \tilde{Z})}(t) \leq \psi_+(t, \Theta)$, where $\psi_+(\cdot, \Theta)$ is convex and $\psi_+(0, \Theta) = \psi'_+(0, \Theta) = 0$. Then, 
    \begin{align} 
        \generror{\mu}{\mathcal{A}^{(d)}} &\geq \frac1n \sum_{i=1}^n \E_{P_\Theta} \left[ \inf_{t \in [0, C_+)} \frac{\DMI{W'}{Z_i}{\Theta} + \psi_+(t, \Theta)}{t} \right] \,. \label{eq:bound_gen_psi_plus}
    \end{align}
\end{theorem}

\begin{proof}[Proof of \Cref{thm:genbound}]
    The generalization error of $\mathcal{A}^{(d)}$ can be written as
    \begin{align}
        \generror{\mu}{\mathcal{A}^{(d)}} &= \frac1n \sum_{i=1}^n \left\{ \E_{P_{W' |\Theta} \otimes P_{\Theta} \otimes \mu} [\ell^\Theta(\tilde{W'}, \tilde{Z_i})] - \E_{P_{W' | \Theta, Z_i} \otimes P_{\Theta} \otimes \mu} [\ell^\Theta(W', Z_i)] \right\} \,. \label{eq:proof_gend}
    \end{align}
    Our final bounds \eqref{eq:bound_gen_psi_plus} and \eqref{eq:bound_gen_psi_minus} result from applying \Cref{lem:adapted_bu} on each term of the sum in \eqref{eq:proof_gend}, \ie~with $X = W'$, $Y = Z_i$ and $f^U(X, Y) = \ell^\Theta(W', Z_i)$.
    
\end{proof}

\subsection{Applications of \Cref{thm:genbound}}

We specify \Cref{thm:genbound} under different sub-Gaussian conditions on the loss. A random variable $X$ is said to be $\sigma$-sub-Gaussian (with $\sigma > 0$) if for any $t \in \R$,
\begin{equation}
    \E[e^{t(X - \E[X])}] \leq e^{\sigma^2 t^2/2} \,.
\end{equation}

\begin{proof}[Proof of \Cref{thm:adapted_xu}]
    Define $h^\Theta(w', s) = (1/n) \sum_{i=1}^n \ell^\Theta(w', z_i)$ for $w' \in \R^d$, $s = (z_1, \dots, z_n) \in \mathrm{Z}^n$ and $\Theta \in \R^{D \times d}$ s.t. $\Theta^\top \Theta = \mI_d$. The generalization error of $\mathcal{A}^{(d)}$ can be written as,
    \begin{align}
        \generror{\mu}{\mathcal{A}^{(d)}} &= \E_{P_{W'|\Theta} \otimes P_{\Theta} \otimes \mu^{\otimes n}} \left[h^\Theta(\tilde{W'}, \tilde{S}_n)\right] - \E_{P_{W' | Z_i, \Theta} \otimes P_{\Theta} \otimes \mu^{\otimes n}} \left[h^\Theta(W', S_n)\right] \,.
    \end{align}
    Since we assume that $\ell^\Theta(w', Z)$ is $\sigma$-sub-Gaussian under $Z \sim \mu$ for all $w'$ and $\Theta$, and $Z_1, \dots, Z_n$ are i.i.d, then $h^\Theta(w', S_n)$ is $\sigma/\sqrt{n}$-sub-Gaussian under $S_n \sim \mu^{\otimes n}$ for all $w'$ and $\Theta$. Therefore, $h^\Theta(\tilde{W}', \tilde{S}_n)$ is $\sigma/\sqrt{n}$-sub-Gaussian under $(\tilde{W}', S_n) \sim P_{W'|\Theta} \otimes \mu^{\otimes n}$ for all $\Theta$, and for $t \in \R$,
    \begin{equation}
        K_{h^\Theta(\tilde{W'}, \tilde{S}_n)}(t) \leq \frac{\sigma^2 t^ 2}{2n} \,.
    \end{equation}
    We conclude by applying \Cref{lem:adapted_bu} with $X = W'$, $Y = S_n$, $U = \Theta$ and $f^U(X,Y) = h^\Theta(W', S_n)$, and the fact that,
    \begin{equation}
        \inf_{t > 0} \frac{\DMI{W'}{S_n}{\Theta} + \sigma^2 t^2 / (2n)}{t} = \sqrt{\frac{2\sigma^2}{n} \DMI{W'}{S_n}{\Theta}} \,.
    \end{equation} 

\end{proof}

\begin{proof}[Proof of \Cref{thm:bounded_loss}]
    Let $\Theta \in \R^{D \times d}$ s.t. $\Theta^\top \Theta = \mI_d$. Since $\ell^\Theta(\tilde{W'}, \tilde{Z})$ is $\sigma_\Theta$-sub-Gaussian under $(\tilde{W'}, \tilde{Z}) \sim P_{W'} \otimes \mu$, then for any $t \in \R$, $K_{\ell^\Theta(\tilde{W}', \tilde{Z})}(t) \leq \sigma_\Theta^2 t^2 / 2$. We conclude by applying \Cref{thm:genbound} and the fact that for $i \in \{1, \dots, n\}$,
    \begin{equation}
        \inf_{t > 0} \frac{\DMI{W'}{Z_i}{\Theta} + \sigma_\Theta^2 t^2 / 2}{t} = \sqrt{2\sigma_\Theta^2 \DMI{W'}{Z_i}{\Theta}} \,.
    \end{equation}
    
\end{proof}

\begin{corollary} \label{cor:bounded_loss} %
    Assume that for any $\Theta \sim P_\Theta$, $\ell^\Theta(\tilde{W}', \tilde{Z}) \leq C$ almost surely. Then,
    \begin{align}
        | \generror{\mu}{\mathcal{A}^{(d)}} | &\leq \frac{C}{n} \sum_{i=1}^n \E_{P_\Theta}\left[\sqrt{\frac{\DMI{W'}{Z_i}{\Theta}}{2}}\,\right] \,.
    \end{align}
\end{corollary}

\begin{proof}[Proof of \Cref{cor:bounded_loss}]
    Since for any $\Theta \sim P_\Theta$, $\ell^\Theta(\tilde{W}', \tilde{Z}) \leq C$ almost surely, then by Hoeffding's lemma, we have for all $t \in \R$,
    \begin{equation}
        \E_{P_{W' | \Theta} \otimes \mu} \big[e^{t \{ \ell^\Theta(\tilde{W}', \tilde{Z}) - \E_{P_{W' |\Theta} \otimes \mu}[\ell^\Theta(\tilde{W}', \tilde{Z})]\}} \big] \leq e^{C^2 t^2/8} \,.
    \end{equation}
    Therefore, $K_{\ell^\Theta(\tilde{W}', \tilde{Z})}(t) \leq C^2 t^2 / 8$. We conclude by applying \Cref{lem:adapted_bu} and the fact that for $i \in \{1, ..., n\}$,
    \begin{equation}
        \inf_{t > 0} \frac{\DMI{W'}{Z_i}{\Theta} + C^2 t^2 / 8}{t} = C\sqrt{\frac{\DMI{W'}{Z_i}{\Theta}}{2}} \,.
    \end{equation} 
    
\end{proof}

\subsection{Tightness of our generalization bounds} \label{app:tightness_sn}

\begin{proposition} \label{prop:tightness_sn}
    For any concave and non-decreasing function $\phi : \R \to \R$, 
    \begin{equation}
        \E_{P_\Theta}\big[\phi\big(\DMI{W'}{S_n}{\Theta} \big)\big] \leq \phi\big(\MI{W}{S_n}\big) \,.
    \end{equation}
\end{proposition}

\begin{proof}[Proof of \Cref{prop:tightness_sn}]
    Let $W \in \mathrm{W}_{\Theta, d}$. Then, $S_n \rightarrow (W', \Theta) \rightarrow W$ and $S_n  \rightarrow W \rightarrow (W', \Theta)$ form two Markov chains, so equality holds in the data-processing inequality, leading to $\MI{W}{S_n} = \MI{W', \Theta}{S_n}$.

    By the chain rule of mutual information, and since $\Theta$ and $S_n$ are independent, 
    \begin{equation}
        \MI{W', \Theta}{S_n} = \MI{\Theta}{S_n} + \CMI{W'}{S_n}{\Theta} = \CMI{W'}{S_n}{\Theta} \,.
    \end{equation}
    Since $\phi$ is non-decreasing,
    \begin{equation}
        \phi\big(\MI{W', \Theta}{S_n}\big) \geq \phi\big(\CMI{W'}{S_n}{\Theta} \big)
    \end{equation}
    Applying the definition of conditional mutual information and Jensen's inequality yields,
    \begin{equation}
        \phi(\CMI{W'}{S_n}{\Theta}) = \phi(\E_{P_\Theta}\big[ \DMI{W'}{S_n}{\Theta} \big]) \geq \E_{P_\Theta}\big[ \phi( \DMI{W'}{S_n}{\Theta}) \big] \,, 
    \end{equation}
    which concludes the proof.
    
\end{proof}

\begin{proposition} \label{prop:tightness_zi}
    For any concave and non-decreasing function $\phi : \R \to \R$, 
    \begin{equation}
        \frac1n \sum_{i=1}^n \E_{P_\Theta}\big[\phi\big(\DMI{W'}{Z_i}{\Theta} \big)\big] \leq \frac1n \sum_{i=1}^n \big[\phi\big(\MI{W}{Z_i}\big)\big] \,.
    \end{equation}
\end{proposition}

\begin{proof}[Proof of \Cref{prop:tightness_zi}]
    Let $i \in \{1, \dots, n\}$. By applying the same proof techniques of \Cref{prop:tightness_sn} with $Z_i$ instead of $S_n$, one has
    \begin{equation} \label{eq:proof_tightness_zi}
        \E_{P_\Theta}\big[\phi\big(\DMI{W'}{Z_i}{\Theta} \big)\big] \leq \phi\big(\MI{W}{Z_i}\big) \,.
    \end{equation}
    
    The final result follows immediately.

\end{proof}

\begin{proposition} \label{prop:tightness_zi_sn}
    For any concave and non-decreasing function $\phi : \R \to \R$,
    \begin{equation}
        \frac1n \sum_{i=1}^n \E_{P_\Theta}\big[\phi\big(\DMI{W'}{Z_i}{\Theta} \big)\big] \leq \E_{P_\Theta} \left[ \phi\left(\frac{\DMI{W'}{S_n}{\Theta}}{n}\right) \right] \,.
    \end{equation}
\end{proposition}

\begin{proof}[Proof of \Cref{prop:tightness_zi_sn}]
    By adapting the proof of \citep[Proposition 2]{bu2019}, one has
    \begin{equation} \label{eq:proof_tightness_zi_sn}
        \frac1n \sum_{i=1}^n \phi\big( \DMI{W'}{Z_i}{\Theta} \big) \leq \phi\left( \frac{\DMI{W'}{S_n}{\Theta}}{n} \right)
    \end{equation}
    The result follows immediately by taking the expectation of \eqref{eq:proof_tightness_zi_sn} and applying the linearity of the expectation on the left-hand side term.
    
\end{proof}

\subsection{Detailed derivations for Gaussian mean estimation} \label{app:gmi}

\paragraph{Problem statement.} The loss function is defined for any $(w, z) \in \R^D \times \R^D$ as $\ell(w, x) = \| w - z \|^2$. Let $Z_1, \dots, Z_n$ be $n$ random variables i.i.d.\ from $\mathcal{N}({\bf0}_D, \mI_D)$. Let $d \leq D$ and $\Theta \sim P_\Theta$ s.t. $\Theta^\top \Theta = \mI_d$. Consider a model $\mathcal{A}^{(d)}$ whose objective is $\argmin_{w \in \setW_{\Theta, d}} \riskn{w}$ where the empirical risk is defined for $w \in \R^D$ as $\riskn{w} = \frac1n \sum_{i=1}^n \| w - Z_i\|^2$. This is equivalent to solving $\argmin_{w' \in \R^d} \driskn{w'}{\Theta}$, where 
\begin{equation} \label{appendix:eq:obj_mean_estimation}
    \forall w' \in \R^d,\;\; \driskn{w'}{\Theta} = \frac1n  \sum_{i=1}^n \| \Theta w' - Z_i\|^2 \,.
\end{equation}

The gradient of \eqref{appendix:eq:obj_mean_estimation} with respect to $w'$ is,
\begin{align}
    \nabla_{w'} \driskn{w}{\Theta} = \frac2n \sum_{i=1}^n \Theta^\top (\Theta w' - Z_i) \,,
\end{align}
and solving $\nabla_{w'} \driskn{w}{\Theta} = 0$ yields $(\Theta^\top \Theta) w' = \Theta^\top \bar{Z}$ where $\bar{Z} \triangleq (1/n) \sum_{i=1}^n Z_i$. Since $\Theta^\top \Theta = \mI_d$, we conclude that the minimizer of \eqref{appendix:eq:obj_mean_estimation} is $W' = \Theta^\top \bar{Z}$.

\paragraph{Generalization error.} We recall that the generalization error of $\mathcal{A}^{(d)}$ is defined as,
\begin{align}
    \generror{\mu}{\mathcal{A}^{(d)}} &= \E[\drisk{W'}{\Theta} - \driskn{W'}{\Theta}]
\end{align}
where the expectation is computed with respect to $P_{W'|\Theta, S_n} \otimes P_\Theta \otimes \mu^{\otimes n}$. Since $W' = \Theta^\top \bar{Z}$, $\generror{\mu}{\mathcal{A}^{(d)}}$ can be written as
\begin{align}
    \generror{\mu}{\mathcal{A}^{(d)}} &= \E_{(S_n, \Theta) \sim \mu^{\otimes n} \otimes P_\Theta}\left[\E_{\tilde{Z} \sim \mu}[\| \Theta \Theta^\top \bar{Z} - \tilde{Z} \|^2] - \frac1n \sum_{i=1}^n \| \Theta \Theta^\top \bar{Z} - Z_i \|^2 \right] \label{app:gen_error_gme}
\end{align}
Since $Z_1, \dots, Z_n$ are $n$ i.i.d.\ samples from $\mathcal{N}(\mathbf{0}_D, \mI_D)$ and $\Theta^\top \Theta = \mI_d$, then $P_{\Theta^\top \bar{Z} | \Theta} = \mathcal{N}({\bf0}_d, (1/n) \mI_d)$ and we have
\begin{align}
    \E_{\mu^{\otimes n} \otimes P_\Theta}[\| \Theta \Theta^\top \bar{Z} \|^2] &= \E_{\mu^{\otimes n} \otimes P_\Theta}[\Tr((\Theta \Theta^\top \bar{Z})^\top (\Theta \Theta^\top \bar{Z}))] \\
    &= \E_{\mu^{\otimes n} \otimes P_\Theta}[\Tr(\bar{Z}^\top \Theta \Theta^\top \Theta \Theta^\top \bar{Z})] \\
    &= \Tr(\E_{\mu^{\otimes n} \otimes P_\Theta}[\Theta^\top \bar{Z} (\Theta^\top \bar{Z})^\top]) \\
    &= \frac{d}{n} \,. \label{appendix:eq:trace_cov_wprime}
\end{align}

For $i \in \{1, \dots, n\}$, $\E[\| Z_i \|^2] = \Tr(\E[Z_i Z_i^\top]) = D$, and
\begin{align}
    \E[(\Theta \Theta^\top \bar{Z})^\top Z_i] &= \frac1n \sum_{j=1}^n \E[Z_j^\top \Theta \Theta^\top Z_i] \\
    &= \frac1n \sum_{j=1}^n \Tr(\E[\Theta^\top Z_i (\Theta^\top Z_j)^\top]) \\
    &= \frac1n \Tr(\E[\Theta^\top Z_i (\Theta^\top Z_i)^\top]) \label{eq:proof_step0} \\
    &= \frac{d}{n} \,. \label{eq:proof_step1}
\end{align}
{Equations \eqref{eq:proof_step0} to \eqref{eq:proof_step1} can be justified as follows. Since $Z_i \sim \mathcal{N}(\mathbf{0}_D, \mathbf{I}_D)$, the conditional distribution of $\Theta^\top Z_i$ given $\Theta$ is $\mathcal{N}(\mathbf{0}_d, \Theta^\top \Theta)$, and $\Theta^\top \Theta = \mathbf{I}_d$ by definition. Therefore, $\mathbb{E}[\Theta^\top Z_i (\Theta^\top Z_i)^\top]=\mathbb{E}[\mathbb{E}[\Theta^\top Z_i (\Theta^\top Z_i)^\top | \Theta]]=\mathbf{I}_d$. We conclude that $\Tr(\mathbb{E}[\Theta^\top Z_i (\Theta^\top Z_i)^\top])=d$. }

We thus obtain,
\begin{align}
    \E[\driskn{W'}{\Theta}] &= \E_{(S_n, \Theta) \sim \mu^{\otimes n} \otimes P_\Theta}\left[\frac1n \sum_{i=1}^n \| \Theta \Theta^\top \bar{Z} - Z_i \|^2 \right] \\
    &= \E_{(S_n, \Theta) \sim \mu^{\otimes n} \otimes P_\Theta}\left[\frac1n \sum_{i=1}^n \| \Theta \Theta^\top \bar{Z} \|^2 - 2(\Theta \Theta^\top \bar{Z})^\top Z_i + \| Z_i \|^2 \right] \label{eq:proof_step2} \\
    &= D - \frac{d}{n} \,. \label{appendix:eq:exp_empirical_risk_d}
\end{align}
{
Indeed, by the linearity of expectation, \eqref{eq:proof_step2} simplifies as 
\begin{equation}
    \E[\widehat{\mathcal{R}}^\Theta_n(W')] = \mathbb{E}_{\mu^\otimes n \otimes P_{\Theta}}[\| \Theta \Theta^\top \bar{Z} \|^2] - \frac{2}{n} \sum_{i=1}^n \mathbb{E}_{\mu^\otimes n \otimes P_{\Theta}}[(\Theta \Theta^\top \bar{Z})^\top Z_i] + \frac1n \sum_{i=1}^n \mathbb{E}_{\mu}[\| Z_i\|^2] \label{eq:proof_step3}    
\end{equation}
Since $(Z_i)_{i=1}^n$ are i.i.d. from $\mathcal{N}(\mathbf{0}_D, \mathbf{I}_D)$, we proved that $\mathbb{E}_{\mu^\otimes n \otimes P_{\Theta}}[\|\Theta \Theta^\top \bar{Z}\|^2]=\frac{d}{n}$ (eq.~\eqref{appendix:eq:trace_cov_wprime}) and $\E[(\Theta \Theta^\top \bar{Z})^\top Z_i] = \frac{d}{n}$ (eq.~\eqref{eq:proof_step1}). Additionally, 
\begin{equation}
    \E_{\mu}[\|Z_i\|^2] = \E_{\mu}[\Tr(\|Z_i\|^2)] = \E_{\mu}[\Tr(Z_i Z_i^\top)] = \Tr(\E_{\mu}[Z_i Z_i^\top]) = \Tr(\mathbf{I}_D) = D
\end{equation}
Plugging these identities in \eqref{eq:proof_step3} yields \eqref{appendix:eq:exp_empirical_risk_d}. 
}

On the other hand, 
\begin{equation}
    \E_{(S_n, \Theta, \tilde{Z}) \sim \mu^{\otimes n} \otimes P_\Theta \otimes \mu}[(\Theta \Theta^\top \bar{Z})^\top \tilde{Z}] = \E[\Theta \Theta^\top \bar{Z}]^\top \E[\tilde{Z}] = 0 \,,
\end{equation} 
therefore,
\begin{align}
    \E[\drisk{W'}{\Theta}] &= \E_{(S_n, \Theta) \sim \mu^{\otimes n} \otimes P_\Theta} \E_{\tilde{Z} \sim \mu}[\| \Theta \Theta^\top \bar{Z} - \tilde{Z} \|^2] \\
    &= \E_{(S_n, \Theta) \sim \mu^{\otimes n} \otimes P_\Theta} \E_{\tilde{Z} \sim \mu}[\| \Theta \Theta^\top \bar{Z} \|^2 - 2 (\Theta \Theta^\top \bar{Z})^\top \tilde{Z}  + \| \tilde{Z} \|^2] \\
    &= D + \frac{d}{n} \,. \label{appendix:eq:exp_true_risk_d}
\end{align}

By plugging \eqref{appendix:eq:exp_empirical_risk_d} and \eqref{appendix:eq:exp_true_risk_d} in \eqref{app:gen_error_gme}, we conclude that $\generror{\mu}{\mathcal{A}^{(d)}} = 2d/n$.

\paragraph{Generalization error bound.} We apply \Cref{thm:genbound} to bound the generalization error. To this end, we need to bound the cumulant generating function of $\ell^\Theta(\tilde{W}', \tilde{Z}) = \| \Theta \Theta^\top \bar{Z} - \tilde{Z} \|^2$ given $\Theta$.

Since $(Z_1, \dots, Z_n, \tilde{Z}) \sim \mu^{\otimes n} \otimes \mu$ with $\mu = \mathcal{N}({\bf0}_D, \mI_D)$, then, given $\Theta$, one has $\Theta^\top \bar{Z} \sim \mathcal{N}(\mathbf{0}_d, (1/n) \mI_d)$ and $(\Theta \Theta^\top \bar{Z} - \tilde{Z}) \sim \mathcal{N}(\mathbf{0}_D, \Sigma_\Theta)$ with $\Sigma_\Theta = \Theta \Theta^\top / n + \mI_D$. Therefore, for $d < D$, $\ell^\Theta(\tilde{W}', \tilde{Z}) = \| \Theta \Theta^\top \bar{Z} - \tilde{Z} \|^2$ is the sum of squares of $D$ dependent Gaussian random variables, which can equivalently be written as
\begin{align}
    \ell^\Theta(\tilde{W}', \tilde{Z}) &= \sum_{k=1}^D \lambda_{\Theta,k} U_{\Theta,k}^2 \,, \label{appendix:eq:loss_mean_estimation_d} \\
    U_{\Theta} &= P \Sigma_\Theta^{-1/2} (\Theta W' - \tilde{Z}) 
\end{align}
where $P \in \R^{D \times D}$ and $\lambda_{\Theta} = (\lambda_{\Theta,1}, \dots, \lambda_{\Theta,D}) \in \R^D$ come from the eigendecomposition of $\Sigma_\Theta$, \ie~$\Sigma_\Theta = P \Lambda P^\top$ with $\Lambda = \text{diag}(\lambda_\Theta)$. %
As a consequence, $U_{\Theta} \sim \mathcal{N}(\mathbf{0}_D, \mI_D)$. Note that, since $\Sigma_\Theta$ is positive definite, $P$ is orthogonal and for any $k \in \{1, \dots, D\}$, $\lambda_{\Theta, k} > 0$.

By \eqref{appendix:eq:loss_mean_estimation_d}, $\ell^\Theta(\tilde{W}', \tilde{Z})$ is a linear combination of independent chi-square variables, each with 1 degree of freedom. Therefore, $\ell^\Theta(\tilde{W}', \tilde{Z})$ is distributed from a generalized chi-square distribution, and its CGF is given by,
\begin{align}
    \forall t \leq \frac12 \min_{k \in \{1, \dots, D\}} \lambda_{\Theta, k},\quad K_{\ell^\Theta(\tilde{W}', \tilde{Z})}(t) &= -t \sum_{k=1}^D \lambda_{\Theta, k} - \frac12 \sum_{k=1}^D \log(1 - 2\lambda_{\Theta, k} t) \\
    &= \frac12 \sum_{k=1}^D [-2\lambda_{\Theta, k}t - \log(1 - 2\lambda_{\Theta, k} t)] \,.
\end{align}
Since for any $s < 0$, $- s - \log(1-s) \leq s^2 / 2$, we deduce that
\begin{align}
    \forall t < 0,\quad K_{\ell^\Theta(\tilde{W}', \tilde{Z})}(t) &\leq \frac12 \sum_{k=1}^D \frac{(2\lambda_{\Theta,k} t)^2}{2} = \|\lambda_\Theta\|^2 t^2 \,. \label{eq:cgf_gmi}
\end{align} 
Since $\text{rank}(\Theta \Theta^\top) = \text{rank}(\Theta^\top \Theta)$ and $\Theta^\top \Theta = \mI_d$, then $\text{rank}(\Theta \Theta^\top) = d$. Moreover, $\Theta \Theta^\top$ and $\Theta^\top \Theta$ share the same non-zero eigenvalues. Therefore, $\Theta \Theta^\top$ has $d$ eigenvalues equal to $1$, and $(D-d)$ eigenvalues equal to $0$, thus $\Theta^\top \Theta / n + \mI_d$ has $d$ eigenvalues equal to $1+1/n$ and  and $(D-d)$ eigenvalues equal to $1$, and 
\begin{equation}
    \|\lambda_\Theta\|^2 = d\left(1+\frac{1}{n}\right)^2 + (D-d) \,. \label{app:eq:norm_lambda}
\end{equation}

By combining \Cref{thm:genbound} with \eqref{eq:cgf_gmi} and \eqref{app:eq:norm_lambda}, we obtain
\begin{equation} \label{app:eq:generror_gme_final}
    \generror{\mu}{\mathcal{A}^{(d)}} \leq \frac2n \sqrt{d\left(1+\frac{1}{n}\right)^2 + (D-d)} \sum_{i=1}^n \E_{P_\Theta}\left[\sqrt{\DMI{W'}{Z_i}{\Theta}}\right]
\end{equation}
Applying Jensen's inequality on \eqref{app:eq:generror_gme_final} and the fact that $W' = \Theta^\top W$ with $W = \argmin_{w \in \R^D} \riskn{w} = \bar{Z}$ finally yields,
\begin{equation}
    \generror{\mu}{\mathcal{A}^{(d)}} \leq \frac2n \sqrt{d\left(1+\frac{1}{n}\right)^2 + (D-d)} \sum_{i=1}^n \sqrt{\kSMI{d}{W}{Z_i}} \,.
\end{equation}

\subsection{Detailed derivations for linear regression} \label{app:linreg}

\textbf{Summary.} Consider $n$ i.i.d.\ samples $(x_1, \dots, x_n)$, $x_i \in \R^D$ and a response variable $y = (y_1, \dots, y_n)$, $y_i \in \mathbb{R}$. The goal of $\mathcal{A}^{(d)}$ is $\min_{w\in\setW_{\Theta, d}} \widehat{\mathcal{R}}_n(w) \triangleq %
    (1/n) \| y - Xw \|^2$, where $X \in \R^{n \times D}$ is the design matrix. We show that if $n \geq D$, then $W' = (\Theta X^\top X \Theta^\top)^{-1} \Theta X^\top y$. Moreover, assume that $X$ is deterministic and $y_i = x_i^\top W^\star + \varepsilon_i$ where $W^\star \in \R^D$ and $(\varepsilon_i)_{i=1}^n$ i.i.d.\ from $\mathcal{N}(0, \sigma^2)$. Then, by applying \Cref{thm:genbound}, we bound $\generror{\mu}{\mathcal{A}^{(d)}}$ by a function of $\MI{\phi(\Theta, X) W}{y_i}$, where $\phi(\Theta, X) \triangleq (\Theta X^\top X \Theta^\top)^{-1} \Theta (X^\top X)$ and $W \triangleq \argmin_{w\in\R^D} \riskn{w}$, which can be interpreted as a generalized SMI with a non-isotropic slicing distribution that depends on the fixed $X$. The corresponding derivations are detailed in the rest of this subsection.

\paragraph{Problem statement.} Consider $n$ i.i.d.\ samples $(x_1, \dots, x_n)$ and a response variable $y = (y_1, \dots, y_n)$, where $x_i \in \R^D$ and $y_i \in \R$. Consder a learning algorithm $\mathcal{A}^{(d)}$ whose objective is $\argmin_{w \in \setW_{\Theta, d}} \riskn{w}$, with
\begin{align}
    \forall w \in \R^D,\quad \widehat{\mathcal{R}}_n(w) = \frac1n \sum_{i=1}^n (y_i - x_i^\top w)^2 = \frac1n \| y - Xw \|^2 \,.
\end{align}
where $X \in \mathbb{R}^{n \times D}$ is the design matrix. This objective is equivalent to finding $W' = \argmin_{w' \in \R^d} \driskn{w'}{\Theta}$, where 
\begin{equation} 
    \forall w' \in \R^d,\;\; \driskn{w'}{\Theta} = \frac1n \| y - X\Theta w' \|^2 \,.
\end{equation}
We assume the problem is over-determined, \ie~$D \leq n$. Solving $\nabla_{w'} \driskn{w'}{\Theta} = 0$ yields
\begin{equation} \label{eq:wprime_star}
    W' =  (\Theta X^\top X \Theta^\top)^{-1} \Theta X^\top y \,.
\end{equation}
On the other hand, we know the solution of $\argmin_{w\in\R^D} \riskn{w}$ is the ordinary least squares (OLS) estimator, given by
\begin{equation} \label{eq:wstar}
    W = (X^\top X)^{-1} X^\top y \,.
\end{equation}
Hence, by \eqref{eq:wstar} with \eqref{eq:wprime_star}, we deduce that
\begin{equation} \label{eq:relation_w_wprime}
    W' = (\Theta X^\top X \Theta^\top)^{-1} \Theta (X^\top X) W
\end{equation}

\paragraph{Generalization error.} In the remainder of this section, we assume that $X$ is deterministic and there exists $W^\star \in \R^D$ such that $y_i = x_i^\top W^\star + \varepsilon_i$ where $(\varepsilon_i)_{i=1}^n$ are i.i.d.\ from $\mathcal{N}(0, \sigma^2)$. By using similar techniques as in \Cref{app:gmi}, one can show that $\generror{\mu}{\mathcal{A}^{(d)}} = 2\sigma^2 d/n$.

\paragraph{Generalization error bound.}

Since $y_i \sim \mathcal{N}(x_i^\top W^\star, \sigma^2)$, and by \eqref{eq:wprime_star}, 
\begin{equation}
    x_i^\top \Theta^\top W' \sim \mathcal{N}(x_i^\top \Theta_X W^\star, \sigma^2 x_i^\top \Theta^\top [\Theta X^\top X \Theta^\top]^{-1} \Theta x_i)
\end{equation}
where $\Theta_X = \Theta^\top (\Theta X^\top X \Theta^\top)^{-1} \Theta (X^\top X) \in \R^{D \times D}$. Therefore, 
\begin{equation}
(\tilde{y}_i - x_i^\top \Theta^\top \tilde{W}') \sim \mathcal{N}(x_i^\top (\mI_D - \Theta_X)W^\star, \sigma^2 (1 + x_i^\top \Theta^\top [\Theta X^\top X \Theta^\top]^{-1} \Theta x_i)) \,,
\end{equation}
and
\begin{equation}
    \ell^\Theta(\tilde{W}', \tilde{y}_i) \sim \sigma_i^2\,\chi^2(1, \lambda_i) \,,
\end{equation}
where $\sigma_i^2 = \sigma^2 (1 + x_i^\top \Theta^\top [\Theta X^\top X \Theta^\top]^{-1} \Theta x_i)$, $\lambda_i = (x_i^\top (\mI_D - \Theta_X)W^\star)^2$ and $\chi^2(k, \lambda)$ denotes the noncentral chi-squared distribution with $k$ degrees of freedom and noncentrality parameter $\lambda$. Hence, the moment-generating function of $\ell^\Theta(\tilde{W}', \tilde{y}_i)$ is
\begin{equation}
    \forall t < \frac1{2\sigma_i^2}, \quad \E\big[e^{t\,\ell^\Theta(\tilde{W}', \tilde{y}_i)}\big] = \frac{e^{(\lambda_i \sigma_i^2 t)/(1 - 2\sigma_i^2 t)}}{\sqrt{1-2\sigma_i^2 t}}
\end{equation}
and its expectation is $\E[\ell^\Theta(\tilde{W}', \tilde{y}_i)] = \sigma_i^2(1 + \lambda_i)$. Therefore, for $t < 1/(2\sigma_i^2)$ and $u_i = 2\sigma_i^2 t$,
\begin{align}
    K_{\ell^\Theta(\tilde{W}', \tilde{y}_i)}(t) &= \frac{\lambda_i u_i}{2(1 - u_i)} - \frac12 \log(1 - u_i) - \frac12(1+\lambda_i) u_i \\
    &= \frac12 \{-\log(1 - u_i) - u_i\} + \frac{\lambda_i u_i^2}{2(1 - u_i)} \,.
\end{align}
Since $- \log(1-x) - x \leq x^2/2$ for $x < 0$, we deduce that for $t < 0$,
\begin{align}
    K_{\ell^\Theta(\tilde{W}', \tilde{y}_i)}(t) &\leq \frac{u_i^2}4 + \frac{\lambda_i u_i^2}{2(1-u_i)} \\
    &= \sigma_i^4 t^2 + \frac{2 \lambda_i \sigma_i^4 t^2}{1- 2 \sigma_i^2 t} \,.
\end{align}

By applying \Cref{thm:genbound}, we conclude that
\begin{equation} \label{app:linreg_genbound}
    \generror{\mu}{\mathcal{A}^{(d)}} \leq \frac1n \sum_{i=1}^n \mathbb{E}_{\Theta}\left[ \inf_{t > 0} \frac{\MI{W'}{y_i} + \sigma_i^4 t^2 \big(1 + 2\lambda_i (1+2\sigma_i^2 t)^{-1}\big)}{t} \right] \,.
\end{equation}

By \eqref{eq:relation_w_wprime}, $W'$ is the projection of $W$ along $\phi(\Theta, X) \triangleq (\Theta X^\top X \Theta^\top)^{-1} \Theta (X^\top X) $. The right-hand side term in \eqref{app:linreg_genbound} can thus be interpreted as a generalized SMI with a non-isotropic slicing distribution that depends on the fixed $X$. 

As $d$ converges to $D$, $\lambda = (\lambda_1, \dots, \lambda_n) \in \R^n$ converges to $\mathbf{0}_n$. Indeed, consider the compact singular value decomposition (SVD) of $X \Theta^\top$, \ie~$X \Theta^\top = U S V^\top$ where $S \in \mathbb{R}^{d \times d}$ is diagonal, $U \in \R^{n \times d}$, $V \in \R^{d \times m}$ s.t. $U^\top U = V^\top V = \mI_d$. By using the pseudo-inverse expression of SVD,
\begin{align}
    X \Theta_X &= X \Theta^\top (\Theta X^\top X \Theta^\top)^{-1} \Theta (X^\top X) \\
    &= U S V^\top V S^{-1} U^\top X \\
    &= U U^\top X 
\end{align}
Therefore, $\sqrt{\lambda} = X (\mI_D - U U^\top ) W^\star$.
Since $U^\top U = \mI_d$ with $U \in \R^{n \times d}$, then $\mI_D - U U^\top$ has $(D-d)$ eigenvalues equal to 1 and $d$ eigenvalues equal to 0. Hence, $\lambda$ converges to $\mathbf{0}_n$ as $d \to D$. 

\section{Postponed Proofs for \Cref{subsec:ratedis}}

\subsection{Proof of \Cref{thm:rateDis,thm:quant_ratedis}}

\begin{proof}[Proof of \Cref{thm:rateDis}]
    By the triangle inequality, for any pair of models $(\mathcal{A}, \mathcal{A'})$,
    \begin{equation}
        | \generror{\mu}{\mathcal{A}} | \leq | \generror{\mu}{\mathcal{A}} - \generror{\mu}{\mathcal{A}'} | + | \generror{\mu}{\mathcal{A}'} | \,. \label{app:triangleq}    
    \end{equation}
    
    Consider $\mathcal{A} : \mathrm{Z}^n \to \mathrm{W}$ and $\mathcal{A}' : \mathrm{Z}^n \to \mathrm{W}_{\Theta, d}$ such that $\mathcal{A}(S_n) = W$ may depend on $\Theta \sim P_\Theta$, and $\mathcal{A}'(S_n) = \Theta (\Theta^\top W)$. On the one hand, by applying \Cref{lem:adapted_bu} with $X = \Theta^\top W$, $Y = Z_i$, $U = \Theta$ and $f^U(X,Y) = \ell^\Theta(\Theta^\top W, Z_i)$, we obtain
    \begin{equation} \label{app:rate_dist_proof_0}
        | \generror{\mu}{\mathcal{A}'} | \leq \frac{C}{n} \sum_{i=1}^n \E_{P_\Theta}\left[\sqrt{\frac{\DMI{\Theta^\top W}{Z_i}{\Theta}}{2}}\,\right] \,.
    \end{equation}

    On the other hand, by the definition of the generalization error, one can show that
    \begin{align}
        | \generror{\mu}{\mathcal{A}} - \generror{\mu}{\mathcal{A}'} | &= | \E[\risk{W} - \riskn{W}] - \E[\drisk{\Theta^\top W}{\Theta} - \driskn{\Theta^\top W}{\Theta}] | \\
        &\leq | \E[\risk{W} - \drisk{\Theta^\top W}{\Theta} ] | + | \E[\riskn{W} - \driskn{\Theta^\top W}{\Theta}] |
    \end{align}
    where the expectations are computed over $P_{W | \Theta, S_n} \otimes P_\Theta \otimes \mu^{\otimes n}$. Additionally,
    \begin{align}
        | \E[\risk{W} - \drisk{\Theta^\top W}{\Theta} ] | &= | \E_{P_{W | \Theta} \otimes P_\Theta \otimes \mu}[\ell(W, Z) - \ell(\Theta \Theta^\top W, Z)] | \label{app:rate_dist_proof_01} \\
        &\leq \E_{P_{W | \Theta} \otimes P_\Theta \otimes \mu} | \ell(W, Z) - \ell(\Theta \Theta^\top W, Z) | \\
        &\leq L \E_{P_{W | \Theta} \otimes P_\Theta} \| W - \Theta \Theta^\top W \| \,, \label{app:rate_dist_proof_02}
    \end{align}
    where \eqref{app:rate_dist_proof_01} follows from the definition of the population risks $\risk{w}$ and $\drisk{\Theta^\top w}{\Theta}$, and \eqref{app:rate_dist_proof_02} results from the assumption that $\ell(\cdot, z) : \mathrm{W} \to \R_+$ is $L$-Lipschitz for all $z \in \mathrm{Z}$.
    
    Using similar arguments, one can show that 
    \begin{align}
        | \E[\riskn{W} - \driskn{\Theta^\top W}{\Theta} ] | &\leq L \E_{P_{W | \Theta} \otimes P_\Theta} \| W - \Theta \Theta^\top W \| \,, 
    \end{align}
    and we conclude that
    \begin{align} \label{app:rate_dist_proof_1}
        | \generror{\mu}{\mathcal{A}} - \generror{\mu}{\mathcal{A}'} | &\leq 2L \E_{P_{W | \Theta} \otimes P_\Theta} \| W - \Theta \Theta^\top W \| \,.
    \end{align}
    The final result follows from bounding \eqref{app:triangleq} using \eqref{app:rate_dist_proof_0} and \eqref{app:rate_dist_proof_1}.
    
\end{proof}

 \begin{proof}[Proof of \Cref{thm:quant_ratedis}]
 Consider $\mathcal{A} : \mathrm{Z}^n \to \mathrm{W}$ and $\mathcal{A}' : \mathrm{Z}^n \to \mathrm{W}_{\Theta, d}$ such that $\mathcal{A}(S_n) = W$ may depend on $\Theta \sim P_\Theta$, and $\mathcal{A}'(S_n) = \Theta \mathcal{Q}(\Theta^\top W)$. Using the same techniques as in the proof of \Cref{thm:rateDis}, we obtain
 \begin{align}
     | \generror{\mu}{\mathcal{A}} | &\leq 2L \E_{P_{W | \Theta} \otimes P_\Theta} \| W - \Theta \mathcal{Q}(\Theta^\top W) \| + | \generror{\mu}{\mathcal{A}'} | \\
     &\leq 2L \E_{P_{W | \Theta} \otimes P_\Theta} \| W - \Theta \mathcal{Q}(\Theta^\top W) \| + C\,\E_{P_\Theta}\left[\sqrt{\frac{\DMI{\mathcal{Q}(\Theta^\top W)}{S_n}{\Theta}}{2n}}\right] \label{eq:proof_quantratedis}
 \end{align}
 where eq.~\eqref{eq:proof_quantratedis} follows from applying \Cref{thm:adapted_xu}. 

 Then, by using the triangle inequality, the fact that $\|\Theta\| = \| \Theta^\top \Theta \| = 1$, and the properties of $\mathcal{Q}$,
 \begin{align}
     &\E_{P_{W | \Theta} \otimes P_\Theta} \| W - \Theta \mathcal{Q}(\Theta^\top W) \| \\
     &\leq \E_{P_{W | \Theta} \otimes P_\Theta} \| W - \Theta \Theta^\top W \| + \E_{P_{W | \Theta} \otimes P_\Theta} \| \Theta \Theta^\top W - \Theta \mathcal{Q}(\Theta^\top W) \| \\
     &\leq \E_{P_{W | \Theta} \otimes P_\Theta} \| W - \Theta \Theta^\top W \| + \E_{P_{W | \Theta} \otimes P_\Theta} \big[ \| \Theta \| \| \Theta^\top W - \mathcal{Q}(\Theta^\top W) \| \big] \\
     &\leq \E_{P_{W | \Theta} \otimes P_\Theta} \| W - \Theta \Theta^\top W \| + \delta \,. 
 \end{align}
 Finally, since $\mathcal{Q}(\Theta^\top W)$ is a discrete random variable and $\| \Theta^\top W \| \leq M$, we use the same arguments as in \Cref{sec:adapted_xu} to bound $\DMI{\mathcal{Q}(\Theta^\top W)}{S_n}{\Theta}$ by $d\log(2M\sqrt{d}/\delta)$.
 
\end{proof}

\subsection{Rate-distortion bounds applied to feedforward neural networks} \label{app:rate_dis_fc}

In the following, we determine the conditions under which feedforward networks meet the assumptions outlined in \Cref{thm:rateDis,thm:quant_ratedis}. Suppose $\setZ = \{(x, y) \in \mathrm{X} \times \{1, \dots, K\}\}$ is the set of feature-label pairs.

For $q \in \mathbb{N}^*$, a $q$-layer feedforward network is characterized by a mapping $f : \mathrm{W} \times \mathrm{X} \to \R^{K}$ such that the output of its $i$-th layer, $X^{(i)}$, satisfies $X^{(1)} \triangleq W^{(1)} X$ and for $i \in \{2, \dots, q\}$, $X^{(i)} \triangleq W^{(i)}(\psi_i(X^{(i-1)}))$, where $\psi_i$ is an activation function applied element-wise, $W^{(i)} \in \R^{d_{out}^{(i)} \times d_{in}^{(i)}}$, and $X$ is the input feature vector.

For any such feedforward network $f$, we denote by $\bar{f} : \mathrm{W}_{\Theta, d} \times \mathrm{X} \to \R^{K}$ a feedforward network with the same architecture as $f$ (\ie, same number of layers and neurons, and same type of activation functions), but with its parameter space restricted to $\mathrm{W}_{\Theta,d}$. Denote by $\bar{X}^{(i)}$ and $\bar{W}^{(i)}$ the output and weight matrix of the $i$-th layer of $\bar{f}$.

\begin{theorem} \label{thm:distortion_fc_outputs}
    Let $f$ be a $q$-layer feedforward neural network. Assume that for $i \in \{2, \dots, q-1 \}$, $\psi_i$ is $\alpha_i$-Lipschitz continuous and $\psi_i({\bf0}) = {\bf0}$.  Assume for $i \in \{1, \dots, q\}$, $\| W^{(i)} \|_2 \leq M$ and $\| \bar{W}^{(i)} \|_2 \leq M$. Then, for $i \in \{1, \dots, q\}$,
    \begin{equation} \label{eq:bound_layer_l}
        \| X^{(i)} - \bar{X}^{(i)} \|_2 \leq M^{i-1} \| X \|_2 \left( \prod_{j=1}^i \alpha_j \right) \sum_{j=1}^i \| W^{(j)} - \bar{W}^{(j)} \|_2 \,,
    \end{equation}
    where $\alpha_1 \triangleq 1$.
\end{theorem}

\begin{proof}
    We prove this result by induction. By definition, $\| X^{(1)} - \bar{X}^{(1)} \|_2 = \| (W^{(1)} - \bar{W}^{(1)}) X \|_2$. Since the spectral norm is consistent with the Euclidean norm,
    $\| X^{(1)} - \bar{X}^{(1)} \|_2 \leq \| W^{(1)} - \bar{W}^{(1)} \|_2 \| X \|_2$, so \eqref{eq:bound_layer_l} is true for $i = 1$.
    
    Now, let $i > 1$ and assume that \eqref{eq:bound_layer_l} holds for $j \in \{1, \dots, i-1\}$. Then,
    \begin{align}
        \| X^{(i)} - \bar{X}^{(i)} \|_2 &= \| W^{(i)} \psi_i(X^{(i-1)}) - \bar{W}^{(i)} \psi_i(\bar{X}^{(i-1)}) \|_2 \\
        &= \| (W^{(i)} - \bar{W}^{(i)}) \psi_i(X^{(i-1)}) + \bar{W}^{(i)} (\psi_i(X^{(i-1)}) - \psi_i(\bar{X}^{(i-1)})) \|_2 \\
        &\leq \| W^{(i)} - \bar{W}^{(i)} \|_2 \| \psi_i(X^{(i-1)}) \|_2 + \| \bar{W}^{(i)} \|_2 \| \psi_i(X^{(i-1)}) - \psi_i(\bar{X}^{(i-1)}) \|_2 \,, \label{eq:induction_00} 
    \end{align}
    where \eqref{eq:induction_00} results from applying the triangle inequality and $\| Mx \|_2  \leq \| M \|_2 \|x\|_2$. Since $\psi_i$ is $\alpha_i$-Lipschitz continuous and $\psi_i({\bf0}) = \bf0$, we obtain
    \begin{align}
        \| X^{(i)} - \bar{X}^{(i)} \|_2 &\leq \alpha_i \left( \| W^{(i)} - \bar{W}^{(i)} \|_2 \| X^{(i-1)} \|_2 + \| \bar{W}^{(i)} \|_2 \| X^{(i-1)} - \bar{X}^{(i-1)} \|_2 \right) \,. \label{eq:induction_0}
    \end{align}
    By recursively using the definition of $X^{(i)}$ and $\| W^{(i)} \psi_i(X^{(i-1)}) \|_2 \leq \alpha_i \| W^{(i)} \|_2  \| X^{(i-1)} \|_2$, one can show that
    \begin{equation}
        \| X^{(i-1)} \|_2 \leq \| X \|_2 \prod_{j=1}^{i-1} \alpha_j \| W^{(j)} \|_2 \leq M^{i-1} \| X \|_2 \prod_{j=1}^{i-1} \alpha_j \,. \label{eq:induction_1}
    \end{equation}
    Additionally, since we assume \eqref{eq:bound_layer_l} holds for $j \in \{1, \dots, i-1\}$,
    \begin{equation}
        \| X^{(i-1)} - \bar{X}^{(i-1)} \|_2 \leq M^{i-2} \| X \|_2 \left( \prod_{j=1}^{i-1} \alpha_j \right) \sum_{j=1}^{i-1} \| W^{(j)} - \bar{W}^{(j)} \|_2 \,.  \label{eq:induction_11}
    \end{equation}
    By plugging \eqref{eq:induction_1} and \eqref{eq:induction_11} in \eqref{eq:induction_0}, we obtain
    \begin{align}
        \| X^{(i)} - \bar{X}^{(i)} \|_2 &\leq \| X \|_2 \left( \prod_{j=1}^{i} \alpha_j \right) \left( M^{i-1} \| W^{(i)} - \bar{W}^{(i)} \|_2 + \| \bar{W}^{(i)} \|_2 M^{i-2} \sum_{j=1}^{i-1} \| W^{(j)} - \bar{W}^{(j)} \|_2  \right)  \label{eq:induction_2} \\
        &\leq M^{i-1} \| X \|_2 \left( \prod_{j=1}^{i} \alpha_j \right) \sum_{j=1}^i \| W^{(j)} - \bar{W}^{(j)} \|_2  \,,
    \end{align}
    which concludes the proof.
    
 \end{proof}

\begin{theorem} \label{thm:rate_dis_bound_fc}
    Let $f$ and $\bar{f}$ be two $q$-layer feedforward neural networks satisfying the assumptions in \Cref{thm:distortion_fc_outputs}. Denote by $\mathcal{A}$ (respectively, $\bar{\mathcal{A}}$) the learning algorithm consisting in training $f$ (resp., $\bar{f}$) using the loss function $\ell : \setW \times \setZ \to \R_+$, where $\setZ = \mathrm{X} \times \{1, \dots, K\}$. Let $\tilde{\ell} : \R^K \times \{1, \dots, K\} \to \R_+$ be the mapping such that for any $w \in \setW$ and $z = (x,y) \in \setZ$, $\ell(w,z) = \tilde{\ell}(f(w,x), y)$ (resp., $\ell(w,z) = \tilde{\ell}(\bar{f}(w,x), y)$). Assume that $\tilde{\ell}$ is $\beta$-Lipschitz w.r.t. the first variable. Suppose additionally that $\forall X \in \mathrm{X}$, $\| X \|_2 \leq R$. Then,
    \begin{equation} 
        | \generror{\mu}{\mathcal{A}} - \generror{\mu}{\bar{\mathcal{A}}} | \leq 2 \beta M^{q-1} R \left( \prod_{i=1}^q \alpha_i \right) \E \Big[ \sum_{i=1}^q \| W^{(i)} - \bar{W}^{(i)} \|_2 \Big] \,.
    \end{equation}
\end{theorem}

\begin{proof}
    By definition of the generalization error,
    \begin{align}
        | \generror{\mu}{\mathcal{A}} - \generror{\mu}{\bar{\mathcal{A}}} | &= | \E[\risk{W} - \riskn{W}] - \E[\risk{\bar{W}} - \riskn{\bar{W}}] | \\
        &= | \E [\risk{W} - \risk{\bar{W}} - (\riskn{W} - \riskn{\bar{W}})] | \\
        &\leq \E [| \risk{W} - \risk{\bar{W}} | + | \riskn{W} - \riskn{\bar{W}} |] \label{proof:rate_dis}
    \end{align}
    For any $(W, \bar{W}) \sim P_{W|S_n} \otimes P_{\bar{W}|S_n}$,
    \begin{align}
        | \risk{W} - \risk{\bar{W}} | &= | \E_{Z\sim\mu}[\ell(W,Z)] - \E_{Z\sim\mu}[\ell(\bar{W},Z)]|  \\
        &\leq | \E_{(X,Y) \sim \mu}[\tilde{\ell}(f(W,X),Y) - \tilde{\ell}(\bar{f}(\bar{W},X),Y)] | \\
        &\leq \E_{(X,Y) \sim \mu} \big[ | \tilde{\ell}(f(W,X),Y) - \tilde{\ell}(\bar{f}(\bar{W},X),Y) | \big] \\
        &\leq \beta\,\E \big[ \| f(W,X) - \bar{f}(\bar{W},X) \|_2 \big] \,. \label{proof:rate_dis_0}
    \end{align}
    We bound \eqref{proof:rate_dis_0} using \Cref{thm:distortion_fc_outputs} and we obtain,
    \begin{align}
    | \risk{W} - \risk{\bar{W}} | &\leq \beta M^{q-1} \E \big[ \| X \|_2 \big] \left( \prod_{i=1}^q \alpha_i \right) \sum_{i=1}^q \| W^{(i)} - \bar{W}^{(i)} \|_2  \\
    &\leq \beta M^{q-1} R \left( \prod_{i=1}^q \alpha_i \right) \sum_{i=1}^q \| W^{(i)} - \bar{W}^{(i)} \|_2 \,. \label{proof:rate_dis_1}
    \end{align}
    Similarly,
    \begin{align}
        | \riskn{W} - \riskn{\bar{W}} | &= \left| \frac1n \sum_{i=1}^n \ell(W,Z_i) - \frac1n \sum_{i=1}^n \ell(\bar{W},Z_i) \right| \\
        &= \left| \frac1n \sum_{i=1}^n \tilde{\ell}(f(W,X_i),Y_i) - \frac1n \sum_{i=1}^n \tilde{\ell}(\bar{f}(\bar{W},X_i),Y_i) \right| \\
        &\leq \frac1n \sum_{i=1}^n | \tilde{\ell}(f(W,X_i),Y_i) - \tilde{\ell}(\bar{f}(\bar{W},X_i),Y_i) | \\
        &\leq \frac{\beta}{n} \sum_{i=1}^n \| f(W,X_i) - \bar{f}(\bar{W},X_i) \|_2 \\
        &\leq \frac{\beta}{n} M^{q-1}  \left( \sum_{i=1}^n \| X_i \|_2 \right) \left( \prod_{i=1}^{q} \alpha_i \right) \sum_{i=1}^q \| W^{(i)} - \bar{W}^{(i)} \|_2  \\
        &\leq \beta M^{q-1} R \left( \prod_{i=1}^{q} \alpha_i \right) \sum_{i=1}^q \| W^{(i)} - \bar{W}^{(i)} \|_2 \,. \label{proof:rate_dis_2}
    \end{align}

    We obtain the final result by plugging \eqref{proof:rate_dis_1} and \eqref{proof:rate_dis_2} in \eqref{proof:rate_dis}.

\end{proof}

\section{Additional Experimental Details for \Cref{sec:exp}}

\subsection{Methodological details} \label{app:methodology}

\paragraph{Architecture for MINE.} In all our experiments, the MI terms are estimated with MINE \cite{belghazi2018mine} based on a fully-connected neural network with one single hidden layer of dimension $100$. The network is trained for $200$ epochs and a batch size of $64$, using the Adam optimizer \cite{kingma2017adam} with default parameters (on PyTorch).

\paragraph{Quantization method.} We use the quantization scheme of \citet{lotfi2022pac} with minor modifications. We learn $c = [c_1,...,c_L] \in \R^L$ quantization levels in $16$-precision during training using the straight through estimator, and quantize the weights $W' = [W_1, \cdots, W_d] \in \R^d$ into $\widehat{W}_i = c_{q(i)}$, where $q(i) = \argmin_{k \in \{1, \dots, L\}} |W_i - c_k|$. Post quantization, arithmetic coding is employed for further compression, to take into account the fact that quantization levels are not uniformly distributed in the quantized weights. Denote by $p_k$ the empirical probability of $c_k$. Arithmetic coding uses at most $\lceil d \times H(p) \rceil + 2$ bits, where $H(p) = - \sum_{k=1}^L p_k \log_2 p_k$. The total bit requirement for the quantized weights, the codebook $c$, and the probabilities $(p_1, \dots, p_L)$ is bounded by $\lceil d \times H(p) \rceil + L \times (16 + \lceil \log_2 d \rceil) + 2$.

\subsection{Additional details and empirical results} \label{appendix:NNs}

\begin{figure}[t]
  \centering
  \begin{subfigure}[b]{\textwidth}
    \centering
    \includegraphics[width=.4\textwidth]{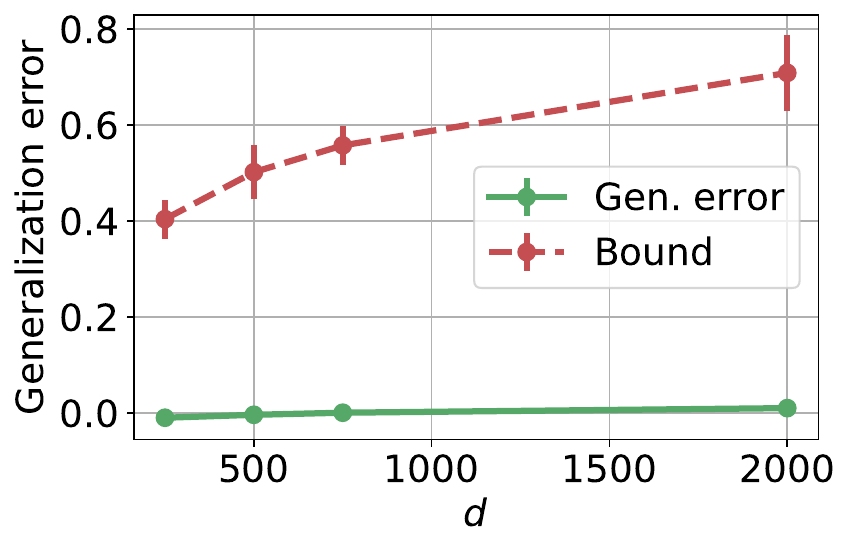}
    \includegraphics[width=.4\textwidth]{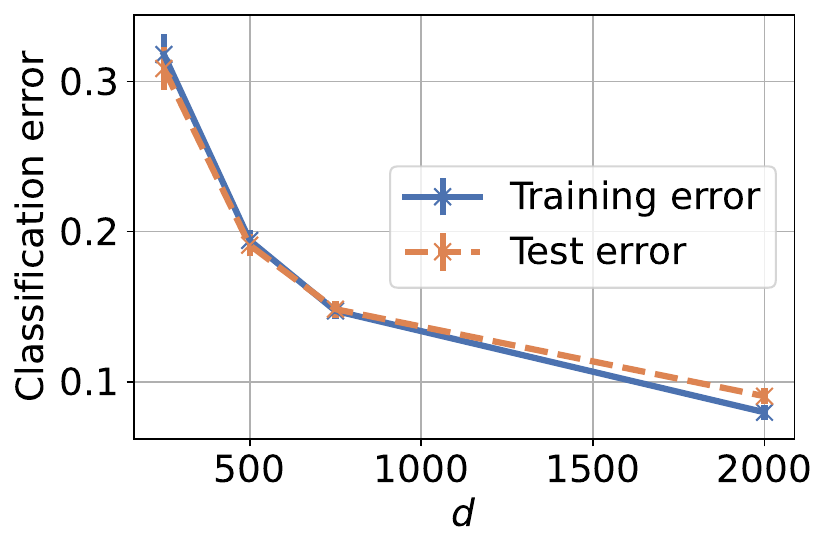}
  \end{subfigure}
  \vspace{-3mm}
  \caption{Generalization bounds on MNIST classification with neural networks trained on $\mathrm{W}_{\Theta, d}$}
  \label{fig:mnist}
\end{figure}

\begin{figure}[t]
  \centering
  \begin{subfigure}[b]{0.4\textwidth}
    \centering
    \includegraphics[width=\textwidth]{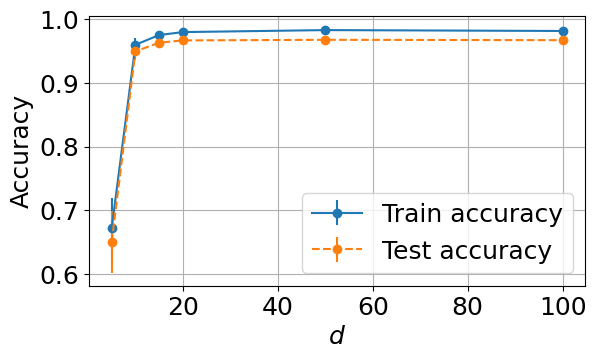}
    \label{fig:image1}
  \end{subfigure}%
  \begin{subfigure}[b]{0.4\textwidth}
    \centering
    \includegraphics[width=\textwidth]{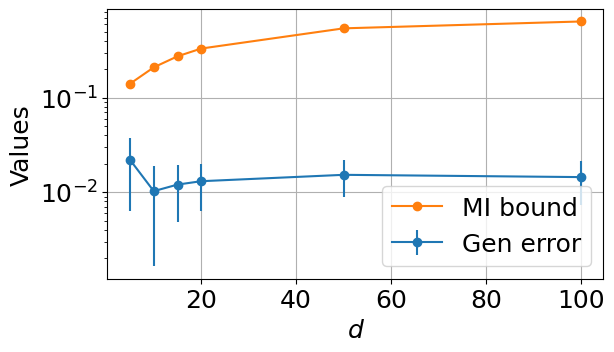}
    \label{fig:image2}
  \end{subfigure}
  \vspace{-3mm}
  \caption{Generalization bounds on Iris dataset classification with neural networks trained on $\mathrm{W}_{\Theta, d}$}
  \label{fig:iris}
\end{figure}

\paragraph{Binary classification with logistic regression (\Cref{subsec:illustration_compress_bounds}).} We consider the binary classification problem solved with logistic regression as described in \citep[{\S}VI]{bu2019}, with features dimension $s = 20$, hence $D = s+1$ (weights and intercept). We train our model on $\mathrm{W}_{\Theta, d}$ for different values of $d < D$, using $n$ training samples. We compute the test error on $\floor{20n/80}$ observations. For each value of $n$ and $d$, we approximate the generalization error for 30 samples of $\Theta$ independently drawn from the SVD-based projector (see \Cref{sec:exp}). We estimate the MI term in the bounds via MINE (with the aforementioned architecture) using 30 samples of $(W', Z_i) \sim P_{W' | S_n, \Theta} \otimes \mu$ for each $\Theta$. 

\paragraph{Classification with NNs (\Cref{subsec:illustration_compress_bounds}).} We consider a feedforward neural network with 2 fully-connected layers of width 200 to classify MNIST \citep{lecun-mnisthandwrittendigit-2010} and CIFAR-10 \citep{krizhevsky2009learning}. The random projections are sampled using the Kronecker product projector, in order to scale better with the high-dimensionality of our models (see \Cref{appendix:NNs}). We train our NNs on $\mathrm{W}_{\Theta, d}$ for different values of $d$, including the intrinsic dimensions reported in \cite{li2018measuring}. We approximate the generalization error for 30 samples of $\Theta$ and estimate our MI-based bounds given by \Cref{thm:genbound}. The MI terms are estimated using MINE over 100 samples of $(W', Z_i) \sim P_{W' | S_n, \Theta} \otimes \mu$ for each $\Theta$. As MINE requires multiple runs, which can be very expensive, we only estimate MI for datasets and models of reasonable sizes: see \Cref{fig:mnist} for results on MNIST. For MNIST and CIFAR-10, we quantize $W'$ and evaluate our quantization-based generalization bounds. To train our NNs, we run Adam \citep{kingma2017adam} with default parameters for 30 epochs and batch size of 64 or 128. 

We also classify the Iris dataset \citep{fisher36lda}. We train a two-hidden-layer NN with width 100 (resulting in $D=10\,903$ parameters) on $\mathrm{W}_{\Theta,d}$. We use Adam with a learning rate of $0.1$ as optimizer, for 200 epochs and batch size of 64. We approximate the generalization error for 20 samples of $\Theta$ independently drawn from the SVD-based projector. We evaluate our generalization bounds (\Cref{thm:bounded_loss}) using MINE over 500 samples of $(W', Z_i) \sim P_{W' | \Theta, S_n} \otimes \mu$ for each $\Theta$. We report results for $d \in \{5, 10, 15, 20, 50, 100\}$ in \Cref{fig:iris}. We obtain over 95\% accuracy at $d=10$ already, and both the best train and test accuracy is achieved for $d=50$. As expected, our bound is an increasing function of $d$ and all of our bounds are non-vacuous.

\paragraph{Influence of the number of quantization levels.} We analyze the influence of the quantization levels $L$ on the generalization error and our bounds in practice. We consider the MNIST classification task with NNs in \Cref{subsec:illustration_compress_bounds} and train for different values of $L$. We report the results in \Cref{fig:quantlevel} for several values of $d$. We observe that for all tested dimensions, the generalization error increases with increasing $L$. Our bound exhibits the same behavior, as anticipated given the dependence on $L$ (see paragraph ``Quantization'' in \Cref{sec:exp}). This experiment illustrates that \emph{(i)} the more aggressive the compression, the better the generalization, \emph{(ii)} our bounds accurately reflect the behavior of the generalization error, and is tighter for lower values of $d$ and $L$. 

\paragraph{Classification with NNs (\Cref{subsec:ratedis_exp}).} We consider a feedforward neural network $f$ with $3$ fully-connected layers and ReLU activations, as formally described in \Cref{app:rate_dis_fc}. We parameterize $f$ with $w = \Theta (\Theta^\top w_1) + \bar{\Theta} (\bar{\Theta}^\top w_2) \in \R^D$, where $\Theta \in \sti(d,D)$ is randomly generated at initialization and $\bar{\Theta} \in \R^{D \times (D-d)}$ is such that $[\Theta, \bar{\Theta}] \in \R^{D \times D}$ forms an orthogonal basis of $\R^D$. The projection matrix $\Theta$ is generated with the sparse projector by \citet{li2018measuring}. Each run consists in randomly selecting a subset of MNIST of $n = 1000$ samples and training $f$ on that dataset for $5$ different samples of $\Theta$. For each $\Theta$, we train for $20$ epochs using the Adam optimizer with a batch size of $256$, learning rate $\eta = 0.01$ for $w_1$ and $\eta/10$ for $w_2$, and other parameters set to their default values \cite{kingma2017adam}. During training, we clamp the norm of each layer's weight matrix at the end of each iteration to satisfy the condition in \Cref{thm:rate_dis_bound_fc}. All hyperparameters, including $C$ and $M$, were chosen so that the neural network trained on MNIST with $d = D$ achieves a training accuracy of at least 99\% in almost all runs.

\begin{figure}[t]
  \centering
    \includegraphics[width=.5\textwidth]{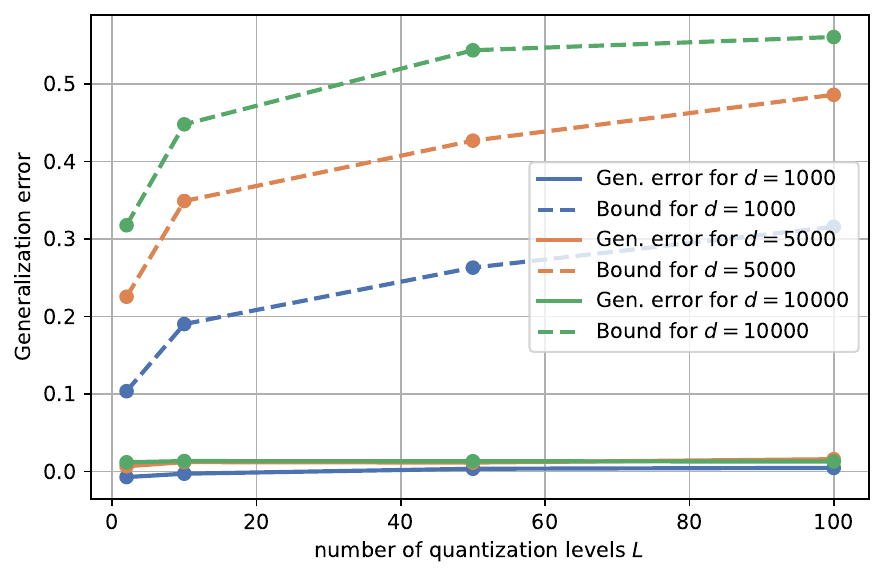}
  \caption{{Influence of the number of quantization levels $L$ on the generalization error and our bounds, for MNIST classification with NNs.}}
  \label{fig:quantlevel}
\end{figure}

\end{document}